\documentclass[journal]{IEEEtran}

\usepackage{amsmath,amssymb,float,arydshln,color}
\usepackage{psfrag,setspace,wrapfig,subfigure}
\usepackage[latin1]{inputenc}
\usepackage{dsfont}
\usepackage[lined,ruled,commentsnumbered]{algorithm2e}
\usepackage{epsfig}
\usepackage{epstopdf}
\usepackage{graphicx}

\usepackage{amsfonts}
\usepackage{cite}
\usepackage{url}

\usepackage{hhline}
\usepackage[table]{xcolor}
\usepackage{multirow}
\usepackage{tabu}
\allowdisplaybreaks

\DeclareMathOperator*{\argmin}{arg\,min}
\newcommand{\bp}{ \begin{proof}}
\newcommand{\ep}{\end{proof} }

\newcommand{\Ex}{\mathbb{E}\hspace{0.05cm}}

\newcommand{\bm}[1]{\mbox{\boldmath $#1$}}

\newcommand{\be}{\begin{equation}}
\newcommand{\ee}{\end{equation}}
\newcommand{\bal}{\begin{align}}
	\newcommand{\eal}{\end{align}}
\newcommand{\bq}{\begin{eqnarray}}
\newcommand{\eq}{\end{eqnarray}}
\newcommand{\bqn}{\begin{eqnarray*}}
	\newcommand{\eqn}{\end{eqnarray*}}
\newcommand{\nn}{\nonumber}
\newcommand{\ba}{\left[ \begin{array}}
	\newcommand{\ea}{\\ \end{array} \right]}
\newcommand{\qd}{\hfill{$\blacksquare$}}
\newcommand{\define}{\;\stackrel{\Delta}{=}\;}

\newcommand{\Tr}{\mbox{\rm {\small Tr}}}
\newtheorem{assumption}{{Assumption}}
\newtheorem{theorem}{{Theorem}}
\newtheorem{corollary}{{Corollary}}

\def\tran{^{\mathsf{T}}}

\def\one{\mathds{1}}

\def\bgamma {{\boldsymbol \gamma}}

\def\h{{\boldsymbol{h}}}

\def\n{{\boldsymbol{n}}}

\def\r{{\boldsymbol{r}}}
\def\s{{\boldsymbol{s}}}

\def\w{{\boldsymbol{w}}}
\def\x{{\boldsymbol{x}}}

\def\real{{\mathbb{R}}}

\makeatletter
\def\hlinewd#1{%
	\noalign{\ifnum0=`}\fi\hrule \@height #1 \futurelet
	\reserved@a\@xhline}
\makeatother

\begin{document}
	\setlength{\abovedisplayskip}{1.6mm}
	\setlength{\belowdisplayskip}{1.6mm}
	
	\title{Performance Limits of Stochastic Sub-Gradient  Learning,\\ Part I: Single Agent Case}
		
	\author{Bicheng~Ying,~\IEEEmembership{Student Member,~IEEE},~and Ali~H.~Sayed,~\IEEEmembership{Fellow,~IEEE}
	\thanks{This work was supported in part by NSF grants CIF-1524250, ECCS-1407712, and DARPA N66001-14-2-4029. A short conference version appears in \cite{ying2016performance}.

	The authors are with Department of Electrical Engineering, University of California, Los Angeles, CA 90095. Emails: \{ybc,sayed\}@ucla.edu.} \vspace{-4mm}}
	
	\maketitle

\begin{abstract}
		In this work and the supporting Part II\cite{ying16ssgd2}, we examine the performance of stochastic sub-gradient learning strategies under weaker conditions than usually considered in the literature. The new conditions are shown to be automatically satisfied by several important cases of interest including SVM, LASSO, and Total-Variation denoising formulations.  In comparison, these problems do not satisfy the traditional assumptions used in prior analyses and, therefore, conclusions derived from these earlier treatments are not directly applicable to these problems. 
		{\color{black} 
		The results in this article establish that stochastic sub-gradient strategies can attain linear convergence rates, as opposed to sub-linear rates, to the steady-state regime.}  A realizable exponential-weighting procedure is employed to smooth the intermediate iterates and guarantee useful performance bounds in terms of convergence rate and excessive risk performance. Part I of this work focuses on single-agent scenarios, which are common in stand-alone learning applications, while Part II\cite{ying16ssgd2} extends the analysis to networked learners. The theoretical conclusions are illustrated by several examples and simulations, including comparisons with the FISTA procedure.
	\end{abstract}
	
	\begin{keywords}
		Sub-gradient, affine-Lipschitz, linear rate, diffusion strategy, SVM, LASSO, Total Variation, FISTA.
	\end{keywords}

	\section{Introduction}%
	The minimization of {\em non-differentiable} convex cost functions is a critical step in the solution of many design problems \cite{bertsekas1999nonlinear,polyak1987introduction,nesterov2004introductory}, including the design of sparse-aware (LASSO) solutions \cite{hastie2009elements,tibshirani1996regression}, support-vector machine (SVM) learners \cite{cortes1995support, shalev2011pegasos,vapnik1998statistical,bishop2006pattern,theodoridis2008Pattern}, or  total-variation-based image denoising solutions \cite{rudin1992nonlinear,beck2009fast}. Several powerful techniques have been proposed in the literature to deal with the non-differentiability aspect of the problem formulation, including methods that employ sub-gradient iterations \cite{bertsekas1999nonlinear,polyak1987introduction,nesterov2004introductory}, cutting-plane techniques \cite{kelley1960cutting}, or proximal  iterations \cite{parikh2013proximal,bach2012optimization}. This work focuses on the class of sub-gradient methods for the reasons explained in the sequel. The sub-gradient technique is closely related to the traditional gradient-descent method\cite{bertsekas1999nonlinear,polyak1987introduction} where the actual gradient is replaced by a sub-gradient at points of non-differentiability. It is one of the simplest methods in current practice but is known to suffer from slow convergence. For instance, it is shown in \cite{nesterov2004introductory} that, for {\em convex} cost functions, the optimal convergence rate that can be delivered by sub-gradient methods in {\em deterministic} optimization problems cannot be faster than $O(1/\sqrt{i})$ under worst case conditions, where $i$ is the iteration index. Under some adjustments to the update equation through the use of weight averaging, it was shown in \cite{lacoste2012simpler} that this rate can be improved to $O(1/i)$.\vspace{-2mm}

	\subsection{The Significance of  Subgradient Algorithms}
	Still, there are at least three strong reasons that motivate a closer examination of the limits of performance of sub-gradient learning algorithms. First, the explosive interest in large-scale and  big data scenarios favors the use of simple and computer-efficient algorithmic structures, of which the sub-gradient technique is a formidable example. Second, it is becoming increasingly evident that more sophisticated optimization iterations do not necessarily ensure improved performance when dealing with complex models and data structures \cite{bousquet2008tradeoffs,bottou2012stochastic,polyak1987introduction,towfic2014stability}. This is because the assumed models, or the adopted cost functions, do not always reflect faithfully the underlying problem structure. In addition, the presence of noise in the data generally implies that a solution that may be perceived to be optimal is actually sub-optimal due to perturbations in the data and models. Third, it turns out that a clear distinction needs to be made between optimizing {\em deterministic} costs \cite{bertsekas1999nonlinear,polyak1987introduction,nesterov2004introductory}, where the cost function is known completely beforehand, and optimizing {\em stochastic} costs, where the cost function is actually unavailable due to its dependence on the unknown probability distribution of the data. Stochastic problem formulations are very common in applications arising in machine learning problems, adaptation, and estimation. We will show that sub-gradient algorithms have surprisingly favorable behavior in the stochastic setting.
	
	Motivated by these remarks, we therefore examine in some detail the performance of {\em stochastic sub-gradient} algorithms for the minimization of non-differentiable convex costs. Our analysis will reveal some interesting properties when these algorithms are used in the context of continuous adaptation and learning (i.e., when actual sub-gradients cannot be evaluated but need to be {\em approximated} continually in an {\em online} manner). The study is carried out for both cases of single stand-alone agents in this part and for multi-agent networks in Part II\cite{ying16ssgd2}. We start with single-agent learning and establish some revealing conclusions about how fast and how well the agent is able to learn. Extension of the results to the multi-agent case will require additional effort due to the coupling that exists among neighboring agents. Interestingly, the same broad conclusions will continue to hold in this case with proper adjustments. \vspace{-2mm}
	
	\subsection{Contributions and Relation to Prior Literature}
	In order to examine the performance of {\it stochastic} sub-gradient implementations, it is necessary to introduce some assumptions on the gradient noise process (which is the difference between a true sub-gradient and its approximation). Here we will {\em diverge} in a noticeable way from assumptions commonly used in the literature for two reasons (see Sec.~\ref{subsection.assump} for further explanations). First, we shall introduce {\em weaker} assumptions than usually adopted in prior works and, secondly and more importantly, we shall show that our assumptions are {\em automatically} satisfied for important cases of interest (such as SVM, LASSO, Total Variation). In contrast, these applications do not satisfy the traditional assumptions used in the prior literature and, therefore, conclusions derived based on these earlier works are not directly applicable to these problems. For example, it is common in the literature to assume that the cost function has a bounded gradient \cite{nemirovski2009robust,boyd2008stochastic,polyak1987introduction,Agarwal2012Information,lacoste2012simpler}; this condition is not even satisfied by quadratic costs whose gradient vectors are affine in their parameter and therefore grow unbounded. This condition is also in direct conflict with strongly-convex costs \cite{Agarwal2012Information}. By weakening the assumptions, the analysis in this work becomes more challenging. At the same time, the conclusions that we arrive at will be stronger and more revealing, and they will apply to a broader class of algorithms and scenarios.

	A second aspect of our study is that we will focus on the use of {\em constant} step-sizes in order to enable continuous adaptation and learning. Since the step-size is assumed to remain constant, the effect of gradient noise is always present and does not die out, as would occur if we were using instead a diminishing step-size of the form $\mu(i)=\tau/i$ for some $\tau>0$ as is common in many other studies\cite{towfic2016excess,shalev2011pegasos,boyd2008stochastic}. Such diminishing step-sizes annihilate the gradient noise term asymptotically albeit at the cost of turning off adaptation in the long run. When this happens, the learning algorithm loses its ability to track drifts in the solution. In contrast, a constant step-size keeps adaptation alive and endows the learning algorithm with an inherent tracking mechanism: if the minimizer that we are seeking drifts with time due, for example, to changes in the statistical properties of the data, then the algorithm will be able to track the new location since it is continually adapting \cite{sayed2011adaptive}. This useful tracking feature, however, comes at the expense of a persistent gradient noise term that never dies out. The challenge in analyzing the performance of learning algorithms in the constant adaptation regime is to show that their feedback mechanism induces a stable behavior that reduces the variance of the gradient noise to a small level and that ensures convergence of the iterates to within a small $O(\mu)$-neighborhood of the desired optimal solution. Moreover, and importantly, it turns out that constant step-size adaptation is not only useful under non-stationary conditions when drifts in the data occur, but it is also useful even under stationary conditions when the minimizer does not vary with time. This is because, as we will see, the convergence {\color{black}towards the steady-state regime} will now be guaranteed to occur at an exponential (i.e., linear rather than sub-linear)  rate, $O(\alpha^i)$ for some $\alpha\in(0,1)$, which is much faster than the $O(1/i)$ rate that would be observed under  diminishing step-size implementation for strongly-convex costs.
	
	A third aspect of our contribution is that it is known that sub-gradient methods are not {\em descent} methods. For this reason, it is customary to employ pocket variables (i.e., the best iterate)\cite{bertsekas1999nonlinear,nesterov2004introductory,shor2012minimization,kiwiel1985methods}  or arithmetic averages \cite{shalev2011pegasos} to smooth out the output. However, as the analysis will reveal, the pocket method is not practical in the {\em stochastic} setting (its implementation requires knowledge of unavailable information), and the use of arithmetic averages \cite{moulines2011non} does not match the convergence rate derived later in Sec. \ref{sect.single_analysis}. We shall therefore propose an alternative weighted averaging scheme with an exponentially-decaying window applied to the weight iterates. Similar, but different, weighting schemes applied to the data directly have been used in other contexts in the design of adaptive\cite{sayed2011adaptive} and reinforcement learning\cite{sutton1998reinforcement} schemes. 
	We shall show that the proposed averaging technique does not degrade convergence performance and is able to match the results derived later in Sec. \ref{sect.single_analysis}.
	
	{\em Notation}: We use lowercase letters to denote vectors, uppercase
	letters for matrices, plain letters for deterministic
	variables, and boldface letters for random variables. We also
	use $(\cdot)^{\sf T}$  to denote transposition, $(\cdot)^{-1}$ for matrix inversion,
	$\mbox{\sf Tr}(\cdot)$ for the trace of a matrix, $\lambda(\cdot)$ for the eigenvalues of
	a matrix, $\|\cdot\|$ for the 2-norm of a matrix or the Euclidean
	norm of a vector, and $\rho(\cdot)$ for the spectral radius of a matrix.
	Besides, we use $ A \geq  B$ to denote that $A - B$ is positive
	semi-definite, and $p\succ 0$ to denote that all entries of vector $p$ are positive.

	\section{Problem Formulation}
	\subsection{Problem Formulation}
	We consider the problem of minimizing a risk function, $J(w):\real^{M}\rightarrow\real$, which is assumed to be expressed as the expected value of some loss function, $Q(w; \x)$, namely,
	\begin{equation}
	w^\star \define \argmin_w J(w), \label{eq.originProb}
	\end{equation}
	where we {\color{black} assume $J(w)$ is strongly convex with its unique minimizer denoted by $w^{\star}$}, and where
	\be J(w)\define \Ex\, Q(w;\x).\ee
	Here, the letter $\x$ represents the random data and the expectation operation is performed over the distribution of this data. Many problems in adaptation and learning involve risk functions of this form, including, for example, mean-square-error designs and support vector machine (SVM) solutions --- see, e.g.,
	\cite{sayed2011adaptive,theodoridis2008Pattern,bishop2006pattern}.  For generality, we allow
	the risk function $J(w)$ to be {\em non-differentiable}. This situation is common in machine learning formulations, e.g., in SVM costs and in regularized sparsity-inducing formulations; examples to this effect are provided in the sequel. 
	
	In this work, we examine in some detail the performance of stochastic {\em sub-gradient} algorithms for the minimization of (\ref{eq.originProb}) and reveal some interesting properties when these algorithms are used in the context of continuous adaptation and learning (i.e., when actual sub-gradients cannot be evaluated but need to be {\em approximated} continually in an online manner). This situation arises when the probability distribution of the data is not known beforehand, as is common in practice. This is because in many applications, we only have access to data realizations but not to their actual distribution. 
	
	\subsection{Stochastic Sub-Gradient Algorithm}
	To describe the sub-gradient algorithm, we first recall that the sub-gradient of a {\color{black} convex} function $J(w)$ at any arbitrary point $w_0$ is defined as any vector $g\in\real^{M}$ that satisfies:
	\be
	J(w)\; \geq\; J(w_0) + g\tran(w-w_0),\;\; \forall w \label{eq.3}
	\ee
	We shall often write $g(w_0)$, instead of simply $g$, in order to emphasize that it is a sub-gradient vector at location $w_0$. We note that sub-gradients are generally non-unique. Accordingly, a related concept is that of the {\it sub-differential} of $J(w)$ at $w_0$, denoted by $\partial J(w_0)$. The sub-differential is defined as the set of all possible sub-gradient vectors at $w_0$:
	\be
	\partial J(w_0) \define \left\{g \;|\; J(w) \geq J(w_0) + g\tran(w-w_0),\;\; \forall w \right\}.
	\ee
	In general, the sub-differential $\partial J(w_0)$ is a set and it will collapse to a single point if, and only if, the cost function is differentiable at $w_0$\cite{nesterov2004introductory}; in that case, the sub-gradient vector will coincide with the actual gradient vector at location $w_0$.
	
	Referring back to problem (\ref{eq.originProb}), the traditional sub-gradient method to minimizing the risk function $J(w)$ takes the form:
	\be
	w_{i} =w_{i-1} - \mu\,  g (w_{i-1}),\;\;\;\;i\geq 0\label{lalkdjh13.13lk1lk3}
	\ee
	where $g(w_{i-1})$ refers to {\color{black} one particular choice of a sub-gradient vector for $J(w)$ at location $w_{i-1}$}, and $\mu>0$ is a small step-size parameter.
	{\color{black}Since sub-gradients are non-unique, in construction (\ref{lalkdjh13.13lk1lk3}), it is assumed that once a form for $g(w)$ is selected, that choice remains invariant throughout the adaptation process. That is, the user selects one choice for $g(w)$ and sticks to it throughout the adaptation process. It is not the case that $g(w)$ can sometimes be chosen in one way in one iteration and then in another way in another iteration (we will illustrate this point in examples given further ahead --- see, e.g., (\ref{AAA})).}

	Now, in the context of adaptation and learning, we usually do not know the exact form of $J(w)$ because the distribution of the data is not known to enable computation of $\Ex Q(w;\x)$ and its gradient vector. As such,  true sub-gradient vectors for $J(w)$ cannot be determined and they will need to be replaced by stochastic approximations evaluated from streaming data; examples to this effect are provided in the sequel in the context of support-vector machines and LASSO sparse designs. Accordingly, we replace the deterministic iteration (\ref{lalkdjh13.13lk1lk3}) by the following stochastic iteration\cite{bertsekas1999nonlinear,nesterov2004introductory,shor2012minimization,kiwiel1985methods}:
	\be
	\w_{i} =\w_{i-1} - \mu\,  \widehat{g} (\w_{i-1}), \label{eq.stoSubGrad1}
	\ee
	where the successive iterates, $\{\w_i\}$, are now random variables (denoted in boldface) and $\widehat{g}(\cdot)$ represents an  approximate sub-gradient vector at location $\w_{i-1}$ estimated from data available at time $i$. The difference between an actual sub-gradient vector and its approximation is referred to as {\em gradient noise} and is denoted by
	\be
	\s_i(\w_{i-1}) \define \widehat{g} (\w_{i-1}) - g (\w_{i-1}).\label{eq.stoSubGrad1.22}
	\ee
	
	\subsection{Examples: SVM and LASSO}\label{sec.example1}
	\noindent To illustrate the construction, we list two examples dealing with support vector machines (SVM) \cite{cortes1995support} and the LASSO problem \cite{tibshirani1996regression}; the latter is also known as the sparse LMS problem or basis pursuit \cite{chen2009sparse,chen1998atomic,hastie2009elements}. We will be using these two problems throughout the manuscript to illustrate our findings. \\

		\noindent {\bf Example 1 (SVM problem).}  The two-class SVM formulation deals with the problem of determining a separating hyperplane, $w\in\real^{M}$, in order to classify feature vectors, denoted by $\h\in\real^{M}$, into one of two classes: $\bm{\gamma}=+1$ or $\bm{\gamma}=-1$. The regularized SVM risk function is of the form:
	\bq
	J^{\rm svm}(w) &\define& \frac{\rho}{2} \| w\|^2 + \Ex \left( \max \left\{0, 1-\bgamma  \h\tran w \right\} \right), \label{eq.expSVM}
	\eq
	where $\rho>0$ is a regularization parameter. We are generally given a collection of independent training data, $\{\bm{\gamma}(i),\h_i\}$, consisting of feature vectors and their class designations and assumed to arise from jointly wide-sense stationary processes. Using this data, the loss function at time $i$ is given by
	\be
	\hspace{-0.5mm}Q^{\rm svm}(w;\{\bgamma(i),\h_i\})\hspace{-0.5mm}=\hspace{-0.5mm} \frac{\rho}{2} \| w\|^2 +  \max \left\{0, 1-\bgamma(i) \h_i\tran w \right\}
	\label{lalkd.1lk3l1k3}\ee
	where the second term on the right-hand side, which is also known as the hinge function, is  non-differentiable at all points $w$ satisfying $1-\bm{\gamma}(i)\h_i\tran w=0$. There are generally many choices for the sub-gradient vector at these locations $w$. One particular choice is:
	{\begin{align}
	{g}^{\rm svm}(\w) 
	=&\; \rho \w + \Ex\bgamma \h\, \mathds{I} [\bgamma\h\tran \w \leq 1],  \label{AAA}   
	\end{align}
	where the indicator function $\mathbb{I}[a]$ is defined as follows:
	\be
	\mathds{I}[a]\;=\;\left\{\begin{array}{ll}1,&\mbox{\rm if statement $a$ is true}\\0,&\mbox{\rm otherwise}\end{array}\right.
	\ee
	The choice (\ref{AAA}) requires the computation of an expectation operator, which is infeasible since the distribution of the data is not known beforehand. One approximation for this particular sub-gradient choice at iteration $i$ is the construction
	\be
	\widehat{g}^{\rm svm}(\w_{i-1}) = \rho \w_{i-1} + \bgamma(i) \h_i\, \mathds{I} [\bgamma(i)\h\tran_i \w_{i-1} \leq 1],  \label{BBB}
	\ee
	where the expectation operator is dropped. We refer to (\ref{BBB}) as an instantaneous approximation for (\ref{AAA}) since it employs the instantaneous realizations $\{\bm{\gamma}(i),\h_i\}$ to approximate the mean operation in (\ref{AAA}). There can be other choices for the true sub-gradient vector at $\w_{i-1}$ and for its approximation. However, it is assumed that once a particular choice is made for the form of $g(\w_{i-1})$, as in (\ref{AAA}), then that and its approximation (\ref{BBB}), remain invariant during the operation of the algorithm. Using (\ref{AAA}) and (\ref{BBB}), the gradient noise process associated with this implementation of the SVM formulation is then given by
	\bq\s^{\rm svm}_i(\w_{i-1})=
	\bgamma(i) \h_i\, \mathds{I} [\bgamma(i)\h\tran_i \w_{i-1} \leq 1]  \nn\\
	-\Ex \bgamma \h\, \mathds{I} [\bgamma\h\tran \w_{i-1} \leq 1].\label{svm.noise}
	\eq
	}
	$\hfill \Box$

	\noindent {\bf Example 2 (LASSO problem).} The least-mean-squares LASSO formulation deals with the problem of estimating a sparse weight vector by minimizing a risk function of the form \cite{donoho2006stable,murakami2010sparse}:\footnote{Traditionally, LASSO refers to minimizing a deterministic cost function, such as $\|y-Ax\|^2+\lambda\|x\|_1$. However, we are interested in stochastic formulations, which motivates (\ref{lkadh,13k1jl3k}).}
	\be
	J^{\rm lasso}(w) \define \frac{1}{2}\Ex\|\bgamma - \h\tran w \|^2 +\delta\|w\|_1,\label{lkadh,13k1jl3k}
	\ee
	where $\delta>0$ is a regularization parameter and $\|w\|_1$ denotes the $\ell_1-$norm of $w$. In this problem formulation, the variable $\bm{\gamma}$ now plays the role of a desired signal, while $\h$ plays the role of a regression vector.  It is assumed that the data are zero-mean wide-sense stationary with second-order moments denoted by
	\be
	\r_{h\gamma}\define \Ex\h\bm{\gamma},\;\;\;R_{h}\define \Ex\h\h\tran.
	\ee
	It is generally assumed that $\{\bm{\gamma},\h\}$ satisfy a linear regression model of the form:
	\be
	\bgamma = \h\tran w^{o} + \n,
	\label{lalkd.13lkl1k3lk}\ee
	where $w^{o}\in\real^{M}$ is the desired unknown sparse vector, and $\n$ refers to an additive zero-mean noise component with finite variance $\sigma_n^2$ and independent of $\h$.  If we multiply both sides of (\ref{lalkd.13lkl1k3lk}) by $\h$ from the left and compute expectations, we find that $w^{o}$ satisfies the normal equations:
	\be
	r_{h\gamma}=R_h w^{o}.
	\ee
	We are again given a collection of independent training data, $\{\bm{\gamma}(i),\h_i\}$, consisting of regression vectors and their noisy measured signals. Using this data, the loss function at time $i$ is given by
	\be
	Q^{\rm lasso}(w;\{\bgamma(i),\h_i\})= \frac{1}{2}(\bgamma(i) - \h_i\tran w )^2   +\delta\|w\|_1,
	\label{lalkd791313}\ee
	where the second term on the right-hand side is again non-differentiable.  
	{\color{black}
		One particular choice for the sub-gradient vector is:
		\begin{align}
		g^{\rm lasso}(\w_{i-1})& = -\Ex\h(\bm{\gamma}-\h\tran \w_{i-1})+ \delta {\rm sgn}(\w_{i-1})\nn\\
		&= - R_h(w^o-\w_{i-1}) + \delta {\rm sgn}(\w_{i-1})
		\label{isthel;alkdak2}\end{align}
		where the notation ${\rm sgn}(a)$, for a scalar $a$, refers to the sign function:
		\be
		\mbox{\rm sgn}[a]=\left\{\begin{array}{ll}+1,\;\;&a > 0\\
		0 \;\;&a > 0\\
		-1,&a<0\end{array}\right.
		\ee
		When applied to a vector $a$, as is the case in (\ref{isthel;alkdak2}), the sgn function is a vector consisting of the signs of the individual entries of $a$. 
		Similar to the previous example, it is infeasible to find the exact sub-graident (\ref{isthel;alkdak2}) since $R_h$ is unknown. Instead, we use the following instantaneous approximation for (\ref{isthel;alkdak2}):
		\begin{align}
		\widehat{g}^{\rm lasso}(\w_{i-1}) &= -\h_i(\bgamma(i) - \h_i\tran \w_{i-1}) + \delta {\rm sgn}(\w_{i-1})\nn\\
		&={\hspace{-0.08cm}}-\h_i\h_i\tran(w^\circ - \w_{i-1}) \hspace{-0.08cm}+ \hspace{-0.08cm}\delta{\rm sgn}(\w_{i-1})\hspace{-0.08cm}-\hspace{-0.08cm}\h_i\n(i).
		\label{isthel;alkdak}
		\end{align}
		It then follows that the gradient noise process in the LASSO formulation is given by
		\be
		\s^{\rm lasso}_i(\w_{i-1})=(R_h-\h_i\h_i\tran)(w^{\circ} -\w_{i-1} ) - \h_i\n(i)\label{lkadhj13l1k3lk1}.
		\ee
	}
	$\hfill \Box$

	\section{Modeling Conditions}\label{subsection.assump}
	In order to examine the performance of the stochastic sub-gradient implementation (\ref{eq.stoSubGrad1}), it is necessary to introduce some assumptions on the gradient noise process. We diverge here from assumptions that are commonly used in the literature for two main reasons. First, we introduce {\em weaker} assumptions than usually adopted in prior works and, secondly and more importantly, we show that our assumptions are automatically satisfied by important cases of interest (such as SVM and LASSO). In contrast, these applications do not satisfy the traditional assumptions used in the literature and, therefore, conclusions derived based on these earlier works are not directly applicable to SVM and LASSO problems. We clarify these remarks in the sequel.
	
	{\color{black} First, we emphasize, as explained above, that it is assumed that the particular construction for the sub-gradient function at location $\w_{i-1}$ remains invariant during the operation of the algorithm, as well as its instantaneous approximation.} 
	
	\begin{assumption}[\sc Conditions on gradient noise]\label{asm.noise} The first and second-order conditional moments of the gradient noise process satisfy the following conditions:
		\bq
		\hspace{-0.8cm}		\Ex [\,s_{i}(\w_{i-1}) \,|\,\bm{\cal F}_{i-1}\,] &=& 0\label{usadkh.13lk1l3k2}, \\
		\hspace{-0.8cm}		\Ex [\,\|s_{i}(\w_{i-1})\|^2 \,|\,\bm{\cal F}_{i-1}\, ] &\leq& \beta^2 \|w^{\star}-\w_{i-1}\|^2 +\sigma^2\label{usadkh.13lk1l3k},
		\eq
		for some constants $\beta^2\geq0$ and $\sigma^2\geq 0$, and where the notation $\bm{\cal F}_{i-1}$ denotes the filtration (collection) corresponding to all past iterates:
		\be \bm{\cal F}_{i-1}\;=\;\mbox{\rm filtration by $\{\w_{j},\;j\leq i-1\}$}.
		\ee\qd
	\end{assumption}
	\noindent Conditions (\ref{usadkh.13lk1l3k2}) and (\ref{usadkh.13lk1l3k}) essentially require that the construction of the approximate sub-gradient vector should not introduce bias and that its error variance should decrease as the quality of the iterate improves. Both of these conditions are sensible and, moreover, they will be shown to be satisfied by, for example, SVM and LASSO constructions.\\
	\begin{assumption}[\sc Strongly-convex risk function]\label{asm.strong}
		The risk function is assumed to be $\eta-$strongly-convex (or, simply, strongly-convex), i.e., there exists an $\eta>0$ such that
		\begin{align}
			J(\theta w_1+(1-\theta)w_2) \;\leq&\; \theta J(w_1) + (1-\theta) J(w_2)  \nn\\
			&\hspace{2mm}- \frac{\eta}{2} \theta(1-\theta)\|w_1-w_2\|^2,
		\end{align}
		for any $\theta\in[0,1]$, $w_1,$ and $w_2$. The above condition is equivalent to requiring \cite{polyak1987introduction}:
		\be
		J(w_1)\geq J(w_2) + g(w_2)\tran(w_1-w_2)+ \frac{\eta}{2} \|w_1-w_2\|^2 \label{assump.strongCVX},
		\ee
		for any $g(\cdot)\in \partial J(w_2)$.
		Under this condition, the minimizer $w^\star$ exists and is unique.
		\hfil\qd
	\end{assumption}
	
	\smallskip
	{\color{black}\noindent  Assumption~\ref{asm.strong} is relatively rare in works on non-differentiable function optimization} because it is customary in these earlier works to focus on studying {\em piece-wise} linear risks; these are useful non-smooth functions but they do not satisfy the strong-convexity condition. In our case, strong-convexity is not a restriction because in the context of adaptation and learning, it is common for the risk functions to include a regularization term, which generally  helps ensure strong-convexity.
	
	\begin{assumption}[\sc Sub-gradient is Affine-Lipschitz]\label{asm.sbg}
		It is assumed that the sub-gradient choice used in (\ref{lalkdjh13.13lk1lk3}) is affine Lipschitz, meaning that there exist constants $c\geq0$ and $d\geq 0$ such that the following property holds:
		\be
		\|g(w_1) -{\color{black}g'}(w_2)\| \;\leq\; c\|w_1 -w_2\| + d, \quad \forall w_1,w_2 \label{assump.subGradient},
		\ee
		for any ${\color{black}g'}(\cdot) \in \partial J(\cdot)$.
		\qd
	\end{assumption}

	It is customary in the literature to use in place of Assumption~\ref{asm.sbg} {\color{black} a more restrictive condition that requires the sub-gradient to be bounded  \cite{bertsekas1999nonlinear,nemirovski2009robust,Agarwal2012Information}, i.e., to require instead of (\ref{assump.subGradient}) that
	\be
	\|g(w)\| \leq d_1,\quad \forall w, \;g(w)\in\partial J(w)\label{lalkd.1lk31lk3lk1k3}.
	\ee
	which is also equivalent to assuming the risk function is Lipschitz:
	\be
		\|J(w_1) - J(w_2)\| \leq d_1 \|w_1 - w_2\|, \;\;\;\;\forall w_1,w_2
	\ee
	Such a requirement does not even hold for quadratic risk functions, whose gradient vectors are affine in $w$ and, therefore, grow unbounded! Even more, it can be easily seen that requirement
	(\ref{lalkd.1lk31lk3lk1k3}) is always conflicted with the strong-convexity assumption. For example, if we set $w_1=w$ and $w_2=w^\star$ in (\ref{assump.strongCVX}), we would obtain:
	\be
	J(w)\geq J(w^\star) +\frac{\eta}{2} \|w-w^\star\|^2.
	\label{relation.1}\ee
	Likewise, if we instead set $w_1=w^\star$ and $w_2=w$ in (\ref{assump.strongCVX}), we would  obtain:
	\be
	J(w^\star)\geq J(w) + g(w)\tran(w^\star-w)+\frac{\eta}{2} \|w-w^\star\|^2.
	\label{relation.2}\ee
	Adding relations (\ref{relation.1})--(\ref{relation.2}) we arrive at the strong monotonicity property:
	\bq
	g(w)\tran (w-w^\star)& \geq& \eta\|w-w^\star\|^2,
	\eq
	which implies, in view of the Cauchy-Schwarz inequality, that
	\be \|g(w)\| \geq \eta\|w-w^\star\| \label{eq.sgb_increase}.
	\ee
	In other words, the strong-convexity condition (\ref{assump.strongCVX}) implies that the sub-gradient satisfies
	(\ref{eq.sgb_increase}); and this condition is in clear conflict with the bounded requirement in (\ref{lalkd.1lk31lk3lk1k3}).\vspace{-0.5mm}
	
	One common way to circumvent the difficulty with the bounded requirement (\ref{lalkd.1lk31lk3lk1k3}) and to ensure that it holds is to restrict the domain of $J(w)$ to some bounded convex set, say, $w\in{\cal W}$, in order to bound its sub-gradient vectors, and then employ a projection-based sub-gradient method (i.e., one in which each iteration is followed by projecting $\w_i$ onto ${\cal W}$). However, this approach has at least three difficulties. First, the unconstrained problem is transformed into a more demanding constrained problem involving an extra projection step. Second, the projection step may not be straightforward to perform unless the set ${\cal W}$ is simple enough. Third,  the bound that results on the sub-gradient vectors by limiting $w$ to ${\cal W}$ can be very loose, which will be dependent on the diameter of the convex set ${\cal W}$.\vspace{-0.5mm}

	For these reasons, we do not rely on the restrictive condition (\ref{lalkd.1lk31lk3lk1k3}) and introduce instead the more relaxed affine-Lipschitz condition (\ref{assump.subGradient}). This condition is weaker than  (\ref{lalkd.1lk31lk3lk1k3}). Indeed, it can be verified that (\ref{lalkd.1lk31lk3lk1k3}) implies (\ref{assump.subGradient}) but not the other way around. To see this, assume (\ref{lalkd.1lk31lk3lk1k3}) holds. Then, using the triangle inequality of norms we have
	\bq
	\|g(w_1)-{\color{black}g'(w_2)}\|& \leq &\|g(w_1)\| +\|{\color{black}g'(w_2)} \|\nn\\
	&\leq& d_1\;+\;d_1\nn\\
	&=&2d_1,
	\eq
	which is a special case of (\ref{assump.subGradient}) with $c=0$ and $d=2d_1$. We now verify that important problems of interest satisfy Assumption \ref{asm.sbg} but {\em not} the traditional condition (\ref{lalkd.1lk31lk3lk1k3}).
	
	\noindent {\bf Example 3 (SVM problem).} We revisit the SVM formulation from Example 1. The risk function (\ref{eq.expSVM}) is strongly convex due to the presence of the quadratic regularization term, ${\frac{\rho}{2}}\|w\|^2$, and since the hinge function $\Ex\max \{0, 1-\bgamma  \h\tran w \}$ is convex. The zero-mean property of the  gradient noise process (\ref{svm.noise}) is obvious in this case. With respect to the variance condition, we note that
	\begin{align}
		\Ex [\|s^{\rm svm}_i(\w_{i-1})\|^2 | \bm{\cal F}_{i-1}]=&\; \Ex \h_{i}\tran \h_{i}\,  \mathds{I} [\bgamma(i)\h_{i}\tran \w_{i-1} \leq 1]\nn\\
		&\; - \|\Ex \bgamma \h\, \mathds{I} [\bgamma\h\tran \w_{i-1} \leq 1]  \|^2 \nn\\
		\leq &\;\Ex \h_{i}\tran \h_{i} \, \mathds{I} [\bgamma(i)\h_{i}\tran \w_{i-1} \leq 1]\nn\\
		\leq &\;\Ex \h_{i}\tran \h_{i}\nn \\
		= &\; \Tr(R_h)\label{use.lkadlkl123},
	\end{align}
	so that Assumption~\ref{asm.noise} is satisfied with $\beta^2=0$ and $\sigma^2=\mbox{\rm Tr}(R_h)$. Let us now verify Assumption~\ref{asm.sbg}. For that purpose, we first note that {\color{black} the sub-differentiable of the SVM risk is given by:
		\bq
			\partial_w J^{\rm svm}(w) &=& \rho w + \partial_{w} \left(\Ex \max\left\{0, 1-\bgamma\h\tran w \right\}\right)\nn\\
			&\stackrel{(a)}{=}& \rho w + \Ex \partial_{w} \left(\max\left\{0, 1-\bgamma\h\tran w \right\}\right)\nn\\
			&\stackrel{(b)}{=}&  \rho w -\Ex \bgamma\h \cdot\mathcal{T}_1(1-\bgamma\h\tran w)
		\eq
		where in step (a) we use the fact that the SVM loss function is continuous convex and,
therefore, we can exchange the order of the sub-differential operation with the expectation
operation\cite[Prop. 2.10]{Wets89stochastic}. In step (b), the operator $\mathcal{T}_1$ is defined by
		\be
		\mathcal{T}_1(x) \define 
		\left\{
		\begin{aligned}
			1, \;\;\;&\;\;\; x>0\\
			[0,1],\;\;\;&\;\;\; x=0\\
			0, \;\;\;&\;\;\; x<0
		\end{aligned}
		\right.
		\ee

	Different choices for the value of ${\cal T}_1(x)$ at the location $x=0$ lead to different sub-gradients. We can therefore express any arbitrary sub-gradient in the form
	\be
		{g^{\rm svm}}'(w) = \rho w -\Ex \bgamma\h \cdot\mathcal{T}'_1(1-\bgamma\h\tran w)\label{feawfew.fesd1}
	\ee
	where the notation $\mathcal{T}'_1(x)$ means that we pick a particular value within the range $[0,1]$ to define the sub-gradient (\ref{feawfew.fesd1}). It now follows that
	\begin{align}\label{subs,aldlk13}
		&\hspace{-8mm}\|g^{\rm svm}(w_1)-{g^{\rm svm}}'(w_2)\|\nn\\
		\leq&\;\rho \|w_1-w_2\|+ \big\| \Ex\bgamma \h\, \mathds{I} [\bgamma\h\tran \w_{i-1}\leq1]\big\|\nn\\
		&\; {} +  \big\|\Ex \bgamma\h \cdot \mathcal{T}'_1(1-\bgamma\h\tran w)\big\|.
	\end{align}
	Note further that:
	\begin{align}
		\big\|\Ex \bgamma\h \cdot \mathcal{T}'_1(1-\bgamma\h\tran w)\big\|^2\leq&\; \Ex\|\bgamma\h \cdot\mathcal{T}'_1(1-\bgamma\h\tran w)\|^2\nn\\
		= &\; \Ex \h\tran\h \cdot{[\mathcal{T}'_1(1-\bgamma\h\tran w)]}^2\nn\\
		\leq & \; \Ex \h\tran\h \nn\\
		=& \Tr(R_h)
	\end{align}
	where the last inequality is because $\h\tran\h $ is non-negative and  $[\mathcal{T}_1'(x)]^2$ is uniformly bounded by one. Substituting into (\ref{subs,aldlk13}) gives
	\be
	\|g^{\rm svm}(w_1)-{g^{\rm svm}}'(w_2)\|\leq \rho \|w_1-w_2\| + 2[\Tr(R_h)]^{1/2},
	\ee
	which is of the same form as (\ref{assump.subGradient}) with parameters $c=\rho$ and $d=2[\Tr(R_h)]^{1/2}$.
	$\hfill \Box$
	}
	\bigskip

	\noindent {\bf Example 4 (LASSO problem).} We revisit the LASSO formulation from Example 2. Under the condition that $R_h>0$, the risk function (\ref{lkadh,13k1jl3k}) is again strongly-convex because the quadratic term,
	$\frac{1}{2}\Ex\|\bm{\gamma}-\h\tran w\|^2$, is strongly convex and the regularization term, $\delta\|w\|_1$, is convex. With regards to the gradient noise process (\ref{lkadhj13l1k3lk1}), it was already shown in Eq.~(3.22) in \cite{sayed2014adaptation} that, conditioned on past iterates, it has zero-mean and its conditional variance satisfies:
	\begin{align}
		\Ex [\|s_i^{\rm lasso}(\w_{i-1})\|^2 | \bm{\cal F}_{i-1}]& \leq  a \| w^\circ-{\w}_{i-1}\|^2 +\sigma_n^2 \Tr(R_h) \nn\\
		&\leq 2a \| w^{\star}-\w_{i-1}\|^2 +\sigma_n^2 \Tr(R_h)\nn\\
		&\hspace{4mm} +2a\|w^\circ -w^\star\|^2,
	\end{align}
	where $a=2\Ex\|R_h - \h_i \h_i\tran \|^2$. It follows that  Assumption~\ref{asm.noise} is satisfied with $\beta^2=2a$ and $\sigma^2=\sigma_n^2 \Tr(R_h)+2a\|w^\circ -w^\star\|^2$. Let us now verify Assumption~\ref{asm.sbg}. {\color{black} For that purpose, we first note that the sub-differential set of the LASSO risk is given by:
	\begin{align}
	\partial_w J^{\rm lasso}(w)	
	&= -R_h(w^o-w) + \delta\cdot {\cal T}_2(w)
	\end{align}
	where the operator $\mathcal{T}_2(x)\in\real^M$ for a vector $x$ is defined as:
	\be
	\mathcal{T}_2(x)
	\hspace{-1.2mm}\define\hspace{-1.2mm}
	\ba{c}
	\mathcal{T}_2(x(1))\\
	\mathcal{T}_2(x(2))\\
	\vdots\\
	\mathcal{T}_2(x(M))
	\ea,\;\;
	\mathcal{T}_2(x(i)) \hspace{-1.2mm}\define \hspace{-1.2mm} 
	\left\{\hspace{-0.8mm}
	\begin{aligned}
		1, \;\;\;&\;\; x(i)>0\\
		[-1,1],\;\;\;&\;\; x(i)=0\\
		-1, \;\;\;&\;\; x(i)<0
	\end{aligned}
	\right.		
	\ee
	Different choices for the value of ${\cal T}_2(x)$ at the locations $x(i)=0$ lead to different sub-gradients. We can therefore express any arbitrary sub-gradient in the form
	\be
	{g^{\rm lasso}}'(w) = -R_h(w^o-w) + \delta \cdot{\cal T}_2'(w)\label{fwei.sfehwi}
	\ee
	where the notation $\mathcal{T}'_2(x)$ means that we pick particular values within the range $[-1,1]$ to define the sub-gradient (\ref{fwei.sfehwi}). It now follows that:
	\begin{align}
		&\hspace{-10mm}\|g^{\rm lasso}(w_1)-{g^{\rm lasso}}'(w_2)\|\nn\\~
		=&\; \| R_{h}w_1-R_{h}w_2+  \delta\cdot\big({\rm sgn}(w_1)-\mathcal{T}'_2(w_2)\big)\| 
	\end{align}
	Observing that the difference between any entries of $\mbox{\rm sgn}(w_1)$ and ${\cal T}_2'(w_2)$ cannot be larger than 2 in magnitude, we get
	\begin{align}
	\|g^{\rm lasso}(w_1)-{g^{\rm lasso}}'(w_2)\|	\leq&\; \|R_h\|\|w_1-w_2\|+2\delta \|\one\|\nn\\
		=&\; \|R_h\|\|w_1-w_2\|+2\delta M^{1/2},
	\end{align}
	}where $\one$ is the column vector with all its entries equal to one. We again arrive at a relation of the
	same form as (\ref{assump.subGradient}) with parameters $c=\|R_h\|$ and $d=2\delta M^{1/2}$.
	\hfill $\Box$

	\section{Performance Analysis}
	We now carry out a detailed mean-square-error analysis of the stability and performance of the stochastic sub-gradient recursion (\ref{eq.stoSubGrad1}) in the presence of gradient noise and for {\em constant} step-size adaptation. In particular, we will be able to show that linear (exponential) convergence can be attained at the rate $\alpha^i$ for some $\alpha \in (0,1)$.

	\subsection{Continuous Adaptation}
	Since the step-size is assumed to remain constant, the effect of gradient noise is continually present and does not die out, as would occur if we were using instead a diminishing step-size, say, of the form $\mu(i)=\tau/i$. Such diminishing step-sizes annihilate the gradient noise term asymptotically albeit at the expense of turning off adaptation in the long run. In that case, the learning algorithm will lose its tracking ability. In contrast, a constant step-size keeps adaptation alive and endows the learning algorithm with a tracking mechanism and, as the analysis will show, enables convergence towards the steady-state regime at an exponential rate, $O(\alpha^i)$, for some $\alpha\in (0,1)$. 
	\subsection{A Useful Bound}
	In preparation for the analysis, we first conclude from (\ref{assump.subGradient}) that the following useful condition also holds, involving squared-norms as opposed to the actual norms:
	\be
	\|g(w_1) - {\color{black}g'(w_2)}\|^2 \leq e^2\|w_1 -w_2\|^2 + f^2, \;\; \forall w_1,w_2,\; {\color{black}g'}\in\partial J \label{assump.sbg2},
	\ee
	where $e^2=2c^2$ and $f^2=2d^2$. This is because
	\bq
	\|g(w_1) - {\color{black}g'(w_2)}\|^2 &\leq& \Big(c\|w_1 -w_2\| + d\Big)^2\nn\\
	&\leq& 2c^2\|w_1 -w_2\|^2 + 2d^2.\label{eq.eAndc}
	\eq
	
	\subsection{Stability and Convergence}\label{sect.single_analysis}
	We are now ready to establish the following important conclusion regarding the stability and performance of the stochastic sub-gradient algorithm (\ref{eq.stoSubGrad1}); the conclusion indicates that the algorithm is stable and converges exponentially fast for sufficiently small step-sizes. But first, we explain our notation and the definition of a ``best'' iterate, denoted by $\w_i^{\rm best}$\cite{nesterov2004introductory}. This variable is useful in the context of sub-gradient implementations because it is known that sub-gradient directions do not necessarily correspond to real ascent directions (as is the case with actual gradient vectors for differentiable functions).
	
	At every iteration $i$, the risk value that corresponds to the iterate $\w_i$ is $J(\w_i)$. This value is obviously a random variable due to the randomness in the data used to run the algorithm. We denote the mean risk value by $\Ex J(\w_i)$. The next theorem examines how fast and how close this mean value approaches the optimal value, $J(w^{\star})$. To do so, the statement in the theorem relies on the {\em best pocket} iterate, denoted by $\w_i^{\rm best}$, and which is defined as follows. At any iteration $i$, the value that is saved in this pocket variable is the past iterate, $\w_j$, that has generated the smallest mean risk value up to that point in time, i.e.,
	\be
	\w^{\rm best}_i\;\define\;\argmin_{0\leq j\leq i}\;\Ex\,J(\w_j).
	\ee
	The statement below then proves that $\Ex J(\w_i^{\rm best})$ approaches {\color{black} a small neighborhood of size $O(\mu)$ around $J(w^{\star})$ exponentially fast:}
	\be
	\lim_{i\rightarrow\infty}\;\Ex\,J(\w_i^{\rm best})\;\leq\;J(w^{\star})\;+\;O(\mu),
	\ee
	where the big-O notation $O(\mu)$ means in the order of $\mu$.
	
	\begin{theorem}[{\sc Single agent performance}]
		Consider using the stochastic sub-gradient algorithm (\ref{eq.stoSubGrad1}) to seek the unique minimizer, $w^{\star}$, of the optimization problem (\ref{eq.originProb}), where the risk function, $J(w)$, is assumed to satisfy Assumptions ~\ref{asm.noise}--\ref{asm.sbg}. If the step-size parameter satisfies (i.e., if it is small enough):
		\be \mu< \frac{\eta}{e^2+\beta^2}\label{llkadkl.13lk1lk3l},\ee
		then it holds that
		{\color{black}
			\be
				\Ex J(\w_i^{\rm best})-J(w^\star) \; \leq \; \xi \cdot \alpha^i \Ex\big\|{\color{black}\w_{0}}-w^\star\big\|^2 + \mu(f^2+\sigma^2)/2, 
				\label{result1}
			\ee
		 That is the convergence of $\Ex\,J(\w_i^{\rm best})$ towards the $O(\mu)-$neighborhood around $J(w^\star)$ occurs at the linear rate,  $O(\alpha^i)$, dictated by the parameter:
			\be
				\alpha \define 1-\mu\eta+\mu^2(e^2+\beta^2)=1-O(\mu)
			\ee
			Condition (\ref{llkadkl.13lk1lk3l}) ensures $\alpha\in(0,1)$.
		}
		In the limit:
		\be
		\lim_{i\to\infty}\Ex J(\w_i^{\rm best}) -J(w^\star) \;\leq\;  \mu(f^2+\sigma^2)/2\label{extendadlkad}.
		\ee
		{\color{black}
			That is, for large $i$, $\Ex J(\w_i^{\rm best})$ is approximately $O(\mu)$-suboptimal.	
		}
		\label{them.121}
		\end{theorem} \vspace{-1mm}
	\bp We introduce the error vector, $\widetilde{\w}_i = w^\star - \w_i$, and use it to deduce from (\ref{eq.stoSubGrad1})--(\ref{eq.stoSubGrad1.22}) the following error recursion:
	\be
	\widetilde{\w}_{i}=\widetilde{\w}_{i-1}+\mu g(\w_{i-1}) + \mu s_{i}(\w_{i-1}).
	\ee
	Squaring both sides and computing the conditional expectation we obtain:
	\begin{align}
	&\hspace{-4mm}\Ex[\|\widetilde{\w}_{i}\|^2\,|\,\bm{\cal F}_{i-1}\,]\nn\\
		=&\; \Ex[\,\|\widetilde{\w}_{i-1}+\mu g(\w_{i-1}) + \mu \s_{i}(\w_{i-1})\|^2 \,|\,\bm{\cal F}_{i-1}\,]\label{subckakjdk13lk13}\\
		\stackrel{(a)}{=}&\;\|\widetilde{\w}_{i-1}+\mu g(\w_{i-1})\|^2\;+\;\mu^2\Ex[\|\s_{i}(\w_{i-1})\|^2\,|\,\bm{\cal F}_{i-1}\,]\nn\\
		=&\;\|\widetilde{\w}_{i-1}\|^2+2\mu g(\w_{i-1})\tran \widetilde{\w}_{i-1} \nn\\
		&\;+  \mu^2\|g(\w_{i-1})\|^2+ \mu^2\Ex[\|\s_{i}(\w_{i-1})\|^2\,|\,\bm{\cal F}_{i-1}\,].
		\end{align}
	In step (a), we eliminated the cross term because, conditioned on $\bm{\cal F}_{i-1}$, the gradient noise process has zero-mean. Now, from the strong convexity condition (\ref{assump.strongCVX}), it holds that
	\be
	g(\w_{i-1})\tran \widetilde{\w}_{i-1} \leq J(w^\star) - J(\w_{i-1}) - \frac{\eta}{2} \|\widetilde{\w}_{i-1}\|^2 \label{eq.54}.
	\ee
	Substituting into (\ref{subckakjdk13lk13}) gives
	\begin{align}
		&\hspace{-4mm}\Ex[\|\widetilde{\w}_{i}\|^2\,|\,\bm{\cal F}_{i-1}\,]\nn\\
		\leq&\;\|\widetilde{\w}_{i-1}\|^2+ 2\mu\left(J(w^\star) - J(\w_{i-1}) - \frac{\eta}{2} \|\widetilde{\w}_{i-1}\|^2\right)+ \nn\\
		&  \;\;\;\;{} \mu^2\|g(\w_{i-1})\|^2+\mu^2\Ex[\|s_{i}(\w_{i-1})\|^2\,|\,\bm{\cal F}_{i-1}\,]\label{eq.mainIneq1}.
	\end{align}
	Referring to (\ref{assump.sbg2}), if we set $w_1=\w_{i-1}$, $w_2=w^\star$, and use the fact that {\color{black} there exists one particular sub-gradient $g'(w)$ satisfying $g'(w^\star)=0$}, we obtain:
	\be
	\|g(\w_{i-1})\|^2 \leq e^2\|\widetilde{\w}_{i-1}\|^2 +f^2.
	\ee
	Substituting into (\ref{eq.mainIneq1}), we get\vspace{-1mm}
	\be
	\Ex[\|\widetilde{\w}_{i}\|^2\,|\,\bm{\cal F}_{i-1}\,]\hspace{12cm}\nn
	\ee\vspace{-0.7cm}
	\begin{align}
		\leq&\;(1-\mu\eta+\mu^2e^2)\|\widetilde{\w}_{i-1}\|^2+2\mu J(w^\star) -  2\mu J(\w_{i-1})\nn\\
		&\;{}+ \mu^2 f^2+\mu^2\Ex[\|s_{i}(\w_{i-1})\|^2\,|\,\bm{\cal F}_{i-1}\,] \nn\\[-2mm]
		\stackrel{(\ref{usadkh.13lk1l3k})}{\leq}&\; (1-\mu\eta+\mu^2(e^2+\beta^2))\|\widetilde{\w}_{i-1}\|^2+2\mu J(w^\star)\nn\\
		&\;{} -  2\mu J(\w_{i-1})+ \mu^2 f^2+\mu^2 \sigma^2 \label{eq.mainIneq2}.
	\end{align}
	Taking expectation again we eliminate the conditioning on $\bm{\cal F}_{i-1}$ and arrive at:
	\begin{align}
		2\mu(\Ex J(\w_{i-1}) - J(w^\star))\hspace{-0.5mm}\leq&(1\hspace{-0.5mm}-\hspace{-0.5mm}\mu\eta\hspace{-0.5mm}+\hspace{-0.5mm}\mu^2(e^2+\beta^2))\Ex\|\widetilde{\w}_{i-1}\|^2\nn\\
		 &\;- \Ex\|\widetilde{\w}_{i}\|^2\hspace{-0.5mm}+\hspace{-0.5mm}\mu^2 (f^2+\sigma^2)
		\label{kjad567llk13lk1l3}
	\end{align}
	To proceed, we simplify the notation and introduce the scalars
	\bq
	a(i) &\define& \Ex J(\w_{i-1}) -J(w^\star)  \label{def.60}\\
	b(i) &\define&  \Ex\| \widetilde\w_{i} \|^2	\label{def.61}\\
	\alpha &\define& 1-\mu\eta+\mu^2(e^2+\beta^2)\label{exp.alpha}\\
	\tau^2 &\define&  f^2+\sigma^2\label{def.63}
	\eq
	Note that since $w^{\star}$ is the unique global minimizer of $J(w)$, then it holds that $J(\w_{i-1})\geq J(w^{\star})$ so that $a(i)\geq 0$ for all $i$.  The variable $a(i)$ represents the average {\em excess risk}. Now, we can rewrite (\ref{kjad567llk13lk1l3}) more compactly as
	\be
	2\mu a(i) \leq \alpha b(i-1) - b(i) + \mu^2 \tau^2.
	\label{retudn,akjdj}\ee
	Iterating over {\color{black}$1 \leq i\leq L$, for some interval length $L$, gives:
	\be
	\sum_{i=1}^{L} \alpha^{L-i} (2\mu a(i) - \mu^2 \tau^2)   \leq \alpha^{L} b(0) - b(L) \;\leq\;\alpha^{L} b(0) \label{eq.mainIter}.
	\ee}Let us verify that $\alpha\in(0,1)$. First,  observe from expression (\ref{exp.alpha}) for $\alpha$ that $\alpha(\mu)$ is a quadratic function in $\mu$. This function attains its minimum at location $\mu^o=\eta/(2e^2+2\beta^2)$. For any $\mu$, the value of $\alpha(\mu)$ is larger than the minimum value of the function at $\mu^o$, i.e., it holds that
	\be
	\alpha \geq 1- \frac{\eta^2}{4(e^2+\beta^2)}.
	\label{kald.1l3kl1k3}\ee
	Now, comparing relations (\ref{eq.sgb_increase}) and (\ref{assump.subGradient}), we find that the sub-gradient vector satisfies:
	\be
	\eta\|w-w^{\star}\|\leq \|g(w)\|\leq c\|w-w^{\star}\|+d,\;\;\;\forall w,
	\ee
	which implies that $\eta\leq c$ since the above inequality must hold for all $w$. It then follows from (\ref{eq.eAndc}) that $e^2>\eta^2$ and from
	(\ref{kald.1l3kl1k3}) that
	\be
	\alpha \geq 1- \frac{\eta^2}{4\eta^2}>0.
	\ee
	In other words, the parameter $\alpha$ is positive. Furthermore, some straightforward algebra using (\ref{exp.alpha}) shows that condition (\ref{llkadkl.13lk1lk3l}) implies $\alpha<1$. We therefore established that $\alpha\in(0,1)$, as desired.
	
	Returning to (\ref{retudn,akjdj}), we note that because the (negative) sub-gradient direction is not necessarily a descent direction,  we cannot ensure that $a(i)<a(i-1)$. However, we can still arrive at a useful conclusion by introducing a pocket variable, denoted by $a^{\rm best}(L)\geq 0 $. This variable saves the value of the smallest increment, $a(j)$, up to time $L$, i.e.,
	\be
	a^{\rm best}(L)\define {\color{black}\min_{1\leq i\leq L}}\;a(i)\;.
	\ee
	Let $\w_{L-1}^{\rm best}$ denote the corresponding iterate $\w_i$ where this best value is achieved.
	Replacing $a(i)$ by $a^{\rm best}(L)$  in (\ref{eq.mainIter}) gives
	{\color{black}
	\bq
	(2\mu a^{\rm best}(L) - \mu^2 \tau^2)&\leq& \alpha^{L} b(0)\left(\sum_{i=1}^{L} \alpha^{L-i} \right)^{-1} \nn\\
	&=&  b(0)\cdot \frac{\alpha^{L}(1-\alpha)}{1-\alpha^{L}}\label{eq.mainIter.2},
	\eq
	or, equivalently,
	\be
	2\mu a^{\rm best}(L)\;\leq\;\mu^2 \tau^2\;+\;b(0)\cdot \frac{\alpha^{L}(1-\alpha)}{1-\alpha^{L}}.\label{eq.81}
	\ee
	{\color{black}Using the definitions of $a^{\rm best}(L)$ and $b(0)$,  from (\ref{def.61})--(\ref{def.63}), we can rewrite (\ref{eq.81}) in the form:
		\begin{align}
			\Ex J(\w^{\rm best}_{L-1}) - J(w^\star) \leq&\;  \alpha^L \frac{(1-\alpha)}{2\mu(1-\alpha^{L})}\Ex\|\w_{0} - w^\star\|^2\nn\\
			&\;{}+ \mu(f^2+\sigma^2)/2
		\end{align}
	}}Taking the limit as $L\rightarrow \infty$, we conclude that
	\be
	\lim_{L\to\infty}\Ex J(\w_L^{\rm best}) -J(w^\star) \leq \mu \tau^2/2 = \mu(f^2+\sigma^2)/2.
	\ee
	\ep

	{\color{black}
	\noindent {\bf Remark \#1}: It is important to note that result  (\ref{result1})  extends and enhances a useful result derived by \cite{Agarwal2012Information} where the following {\em lower bound}  was established (using our notation):
	\be
	\Ex J(\w_i) - J(w^{\star}) \geq \frac{\zeta_1}{i} \left[\Ex \|\w_{0}-w^{\star}\|^2\right]\label{eq.73}
	\ee
	for some constant $\zeta_1>0$. This result shows that the convergence of $\Ex J(\w_i)$ towards $J(w^{\star})$ cannot be better than a sub-linear rate when one desires convergence towards the exact minimum value.  In contrast, our analysis that led to (\ref{result1}) establishes the following {\em upper bound}:
	\be
	\Ex J(\w_i) - J(w^{\star}) \leq \zeta \alpha^i \left[\Ex \|\w_{0}-w^{\star}\|^2\right]\;+\;O(\mu)\label{eq.74}
	\ee
	for some constant $\zeta >0$. Observe that this expression is showing that $\Ex J(\w_i)$ can actually approach a small $O(\mu)-$neighborhood around $J(w_i)$ {\em exponentially fast}  at a rate that is dictated by the scalar $0<\alpha<1$. It is clear that the two results (\ref{eq.73}) and (\ref{eq.74}) on the convergence rate  do not contradict each other. On the contrary, they provide complementary views on the convergence behavior (from below and from above). Still, it is useful to remark that the analysis employed by \cite{Agarwal2012Information})  imposes a stronger condition on the risk function than our condition: they require the risk function to be Lipschitz continuous. In comparison, we require the subgradient (and not the risk function) to be affine Lipschitz, which makes the current results applicable to a borader class of problems. 
	
	Figure \ref{fig.rate} illustrates one possible situation of sublinear and linear convergence rate bounds (\ref{eq.73}) and (\ref{eq.74}). The red curve shows
	that there exists an exponential rate $\alpha\in[0,1)$ that bounds the convergence behavior from above
	towards a steady-state regime. The blue line shows a sublinear curve bounding convergence from below.
	\qd\vspace{-2mm}
	\begin{figure}[h]
		\epsfxsize 9.cm \epsfclipon
		\begin{center}
			\leavevmode
			\includegraphics[scale=0.45]{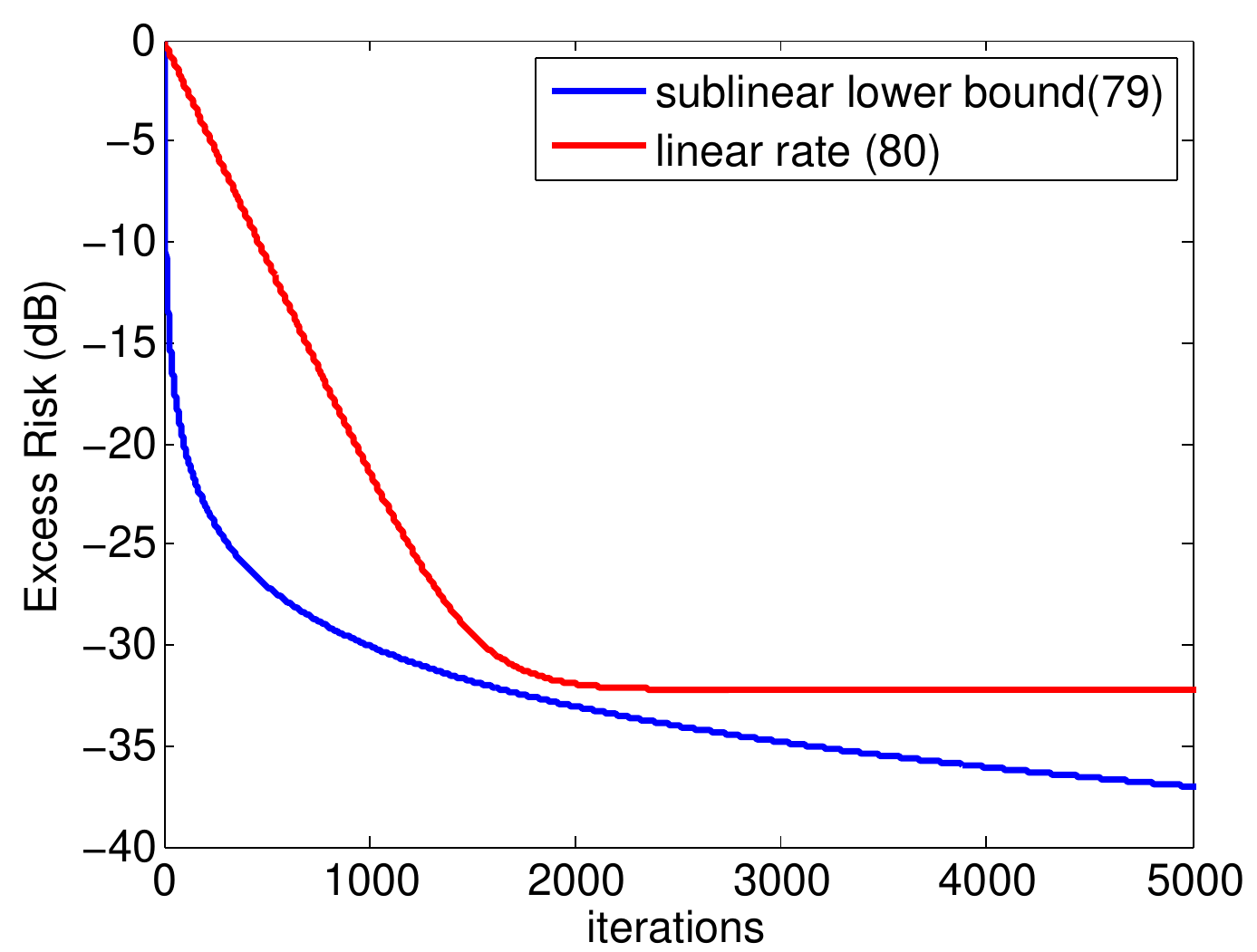} 
			\caption{{Illustration of sublinear and linear convergence bounds (\ref{eq.73}) and (\ref{eq.74}).}}\label{fig.rate}\vspace{-5mm}
		\end{center}
	\end{figure}
%

	\noindent {\bf Remark \#2}: We can similarly comment on how our result (\ref{result1}) relates to the analysis in \cite{NedicThesis}. Using our notation, the main conclusion that follows from Propositions 3.2 and 3.3 in \cite{NedicThesis} is that 
	\be
	\Ex J(\w^{\rm best}_i) - J(w^\star)\leq \frac{\zeta_2}{i} \Ex\|{\color{black}\w_{0}} - w^\star\|^2 + O(\mu)\label{eq.nedic}
	\ee
	for some constant $\zeta_2>0$. This result only ensures a {\em sub-linear} rate of convergence. In contrast, our convergence analysis leads to (\ref{eq.74}), 
	which shows that convergence actually occurs at a linear rate. This conclusion is clearly  more powerful. 
	Furthermore, as was the case with the treatment in \cite{Agarwal2012Information}, the result (\ref{eq.nedic}) is derived in \cite{NedicThesis} by assuming the sub-gradient vectors are bounded, which is a stronger condition than the affine Lipschitz condition used in the current manuscript. \hfill
	\qd
	}\vspace{-3mm}

	\subsection{Exponential Smoothing}
	Theorem \ref{them.121} only clarifies the performance of the best pocket value, which is not readily available during the algorithm implementation since the risk function itself cannot be evaluated.  That is, $J(\w_{i})$ cannot be computed because $J(w)$ is not known due to the lack of knowledge about the probability distribution of the data. However, a more practical conclusion can be deduced from the statement of the theorem as follows. Introduce the geometric sum:{\color{black}
	\be
	S_L\define \sum_{j=0}^L \alpha^{L-j}=\alpha S_{L-1}+1\;=\;\frac{1-\alpha^{L+1}}{1-\alpha},
	\ee}as well as the normalized and convex-combination coefficients:
	\be
	r_L(j) \define\frac{\alpha^{L-j} }{S_L},\;\;\;{\color{black}j=0,1,\ldots,L.} \label{fweofiji23e}
	\ee
	Using these coefficients, we define the weighted iterate
	\begin{align}
	\hspace{-2mm}\bar\w_L\hspace{-1mm}\define& \sum_{j=0}^L r_L(j)\w_{j}\nn\\
	=&\;\;\frac{1}{S_L}\left[\alpha^{L}\w_0+\alpha^{L-1}\w_1+\ldots+\alpha \w_{L-1}+\w_L\right] \label{eq.iter.weight}
	\end{align}
	Observe that, in contrast to $\w_{L}^{\rm best}$, the above weighted iterate is computable since its value depends on the successive iterates $\{\w_j\}$ and these are available during the operation of the algorithm. Observe further that $\bar{\w}_L$ satisfies the recursive construction:
	\be
	\bar{\w}_L\;=\;\left(1-\frac{1}{S_L}\right)\bar{\w}_{L-1}\;+\;\frac{1}{S_L}\w_L.
	\ee
	In particular, as $L\rightarrow\infty$, we have $S_L\rightarrow 1/(1-\alpha)$, and the above recursion simplifies in the limit to
	\be
	\bar{\w}_L\;=\;\alpha \bar{\w}_{L-1}\;+\;(1-\alpha)\w_L.
	\ee

	Now, since $J(\cdot)$ is a convex function, it holds that
	\be
	J(\bar\w_L)\;=\;{\color{black}J\left(\sum_{j=0}^L r_L(j)\w_j\right)\;\leq\;\sum_{j=0}^{L} r_L(j) J(\w_j).}
	\label{akdkkad.ll1k3}\ee
	Using this fact, the following corollary derives a result similar to (\ref{extendadlkad}) albeit applied to $\bar{\w}_L$.
	
	\begin{corollary}[{\sc Weighted iterate}] Under the same conditions as in Theorem~\ref{them.121}, it holds that
		\be
		\lim_{L\to\infty}\Ex J(\bar\w_L) -J(w^\star) \leq  \mu(f^2+\sigma^2)/2\label{lkadj13123},
		\ee
		and the convergence of $\Ex J(\bar\w_L)$ towards $J(w^{\star})$ continues to occur at the same exponential rate, $O(\alpha^L)$.
		\label{coolad.12}\end{corollary}
	\begin{proof}
	We start from (\ref{eq.mainIter}), namely,
	\be
	\hspace{-1mm}	\sum_{i=1}^{L} \alpha^{L-i} (2\mu \Ex J(\w_{i-1})-2\mu J(w^{\star}) - \mu^2 \tau^2) \leq \alpha^{L} b(0) \label{eq.mainIterfff}
	\ee
	and divide both sides by the same sum:
	\begin{align}
		&\sum_{i=1}^{L} \left(\frac{\alpha^{L-i}}{S_{L-1}}\right) (2\mu \Ex J(\w_{i-1})-2\mu J(w^{\star}) - \mu^2 \tau^2) \nn\\
		 &\;\hspace{14mm}\leq
	\left(\frac{\alpha^{L}}{S_{L-1}}\right) b(0) \label{eq.asxacmainIterfff},
	\end{align}
	which gives
	\begin{align}
	&\hspace{-5mm}\sum_{i=1}^{L} r_{L-1}(i-1) (2\mu \Ex J(\w_{i-1})-2\mu J(w^{\star}))\nn\\[-2.2mm]
	&\hspace{2.5cm}\,\leq
	\frac{\alpha^{L}(1-\alpha)}{1-\alpha^{L}} b(0)\;+\;\mu^2\tau^2 \label{eq.asxacmainIterffasxcasf}.
	\end{align}
	Appealing to the convexity property (\ref{akdkkad.ll1k3}) we conclude that
	\be
	2\mu(\Ex J(\bar{\w}_{L-1})-  J(w^{\star}))
	\leq {\color{black}\frac{\alpha^{L}(1-\alpha)}{1-\alpha^{L}} b(0)\;}+\;\mu^2\tau^2 \label{eq.asxacmainIterffasxcasxasxasf}.
	\ee
	Taking the limit as $L\rightarrow\infty$ leads to (\ref{lkadj13123}).
	\end{proof}\vspace{-5mm}
	
	\subsection{Listing of Algorithm}\vspace{-1mm}
	It is interesting to compare result (\ref{lkadj13123}) with what happens in the case of {\em differentiable} risk functions. In that case, the standard stochastic gradient algorithm, using the actual gradient vector rather than sub-gradients, can be employed to seek the minimizer, $w^{\star}$. It was established in
	\cite[Ch. 4]{sayed2014adaptation} that for risk functions that are twice-differentiable, the stochastic gradient algorithm guarantees
	\be
	\lim_{L\to\infty}\Ex J(\w_L) - J(w^\star) = \frac{\mu}{4}\sigma^2,
	\label{llkad.lk1l3kl1k3}\ee
	where the right-hand side is dependent on $\sigma^2$ alone; this factor arises from the bound (\ref{usadkh.13lk1l3k}) on the gradient noise process. In contrast,  in the non-smooth case (\ref{lkadj13123}), we established here a similar bound that is still in the order of $O(\mu)$. However, the size of the bound is not solely dependent on $\sigma^2$ anymore but it also includes the factor $f^2$; this factor arises from condition (\ref{assump.sbg2}) on the sub-gradient vectors. That is, there is some minor degradation (since $\mu$ is small) that arises from the non-smoothness of the risk function. If we set $f=0$ in (\ref{lkadj13123}), we recover (\ref{llkad.lk1l3kl1k3}) up to a scaling factor of $2$. Although the bound in this case is still $O(\mu)$, as desired, the reason why{} it is not as tight as the bound derived in the smooth case in \cite{sayed2014adaptation} is because the derivation in the current paper is not requiring the risk function to be twice differentiable, as was the case in \cite{sayed2014adaptation}, and we are also discarding the term $b(L)$ in equation (\ref{eq.mainIter}). The important conclusion to note is that the right-hand side of (\ref{lkadj13123}) is also $O(\mu)$, as in the smooth case (\ref{llkad.lk1l3kl1k3}).
	
	Using $\alpha$ as a scaling weight in (\ref{eq.iter.weight}) may still be inconvenient because its value needs to be determined. The analysis however suggests that we may replace $\alpha$ by any parameter $\kappa$ satisfying $\alpha \leq \kappa < 1$. The parameter $\kappa$ plays a role similar to the step-size, $\mu$: both become parameters selected by the designer. Next, we introduce the new weighted variable (we continue to denote it by $\bar{w}_L$ as in (\ref{eq.iter.weight}) to avoid a proliferation of symbols; we also continue to denote the scaling coefficients in (\ref{awfwe23}) below by $r_{L}(j)$ similar to (\ref{fweofiji23e})):
	\be
		{\color{black}	\bar\w_{L}\define \sum_{j=0}^L r_L(j)\w_{j} \label{Eq.83},}\vspace{-3mm}
	\ee
	where now \be r_L(j) = \kappa^{L-j}/S_L, \;\; {\color{black}j = 0,1,\ldots,L}\label{awfwe23}\ee
	and $\color{black} S_L=\sum_{j=0}^L\kappa^{L-j}$.
	
	\begin{corollary}[{\sc Relaxed Weighted iterate}]\label{corollary.2} Under the same conditions as in Theorem~\ref{them.121} and $\alpha\leq\kappa<1$, relation (\ref{lkadj13123}) continues to hold with $\bar{\w}_{L}$ in (\ref{eq.iter.weight}) replaced by (\ref{Eq.83}):
	\begin{align}
		\lim_{L\to\infty}\Ex J(\bar\w_L) -J(w^\star) \leq  \mu(f^2+\sigma^2)/2 \label{eq.ff239ij3}
	\end{align}
	And convergence now occurs at the exponential rate $O(\kappa^L)$.
	\end{corollary}
		\bp
		The argument requires some modification relative to what we have done before. We start from (\ref{retudn,akjdj}) again:
		\be
		2\mu a(i) \leq \alpha b(i-1) - b(i) + \mu^2\tau^2.
		\ee
		
		\noindent But unlike the previous derivation in (\ref{eq.mainIter}), now we use $\kappa$ to expand the recursion from {\color{black} iteration $i=1$ to $L$:
		\begin{align}&\hspace{-4mm}\sum_{i=1}^L \kappa^{L-i}(2\mu a(i) -\mu^2\tau^2)\nn\\
		\leq&\; \sum_{i=1}^L \kappa^{L-i}\Big(\alpha b(i-1) - b(i)\Big) \nn\\
		=&\;\sum_{i=0}^{L-1} \kappa^{L-i-1}(\alpha-\kappa) b(i) + \kappa^{L} b(0) - b(L) \nn\\
		\leq&\;\kappa^{L} b(0),
		\end{align}}where in the last inequality we used the fact that $\kappa\geq\alpha$. We can now proceed from here and complete the argument as before.
		\ep
	
	\smallskip
	{\color{black}
	\begin{corollary}[{\sc Mean-Square-Deivation Performance}]\label{corollary.3} Under the same conditions as in Theorem~\ref{them.121} and $\alpha\leq\kappa<1$. It holds for $\bar{\w}_L$ in (\ref{Eq.83}) that
	\be
		\lim_{L\to\infty} \Ex\| \bar{\w}_L - w^\star \|^2 \leq \mu (f^2+\sigma^2)/\eta\label{eq.91}
	\ee
	Moreover, convergence to the steady-state regime occurs at the exponential rate $O(\kappa^L)$.
	\end{corollary}
	\bp
		Referring to equation (\ref{relation.1}), we get
		\be
			\Ex\|\bar{\w}_i-w^\star\|^2 \leq \frac{2}{\eta} \left(\Ex J(\bar{\w}_i) - J(w^\star) \right)
		\ee
		Combining with the corollary \ref{corollary.2}, we arrive at (\ref{eq.91}). 
	\ep
	}

	For ease of reference, we summarize in the table below the listing of the stochastic subgradient learning algorithm with exponential smoothing for which results (\ref{eq.ff239ij3}) and (\ref{llkad.lk1l3kl1k3}) hold.
	
	\noindent\rule{0.49\textwidth}{1pt}\vspace{-2mm}
	\begin{center} \bf \small Stochastic subgradient learning with exponential smoothing\end{center} \vspace{-4mm}
	\rule{0.49\textwidth}{1pt}
	\mbox{\bf \small Initialization}:
	\(
	S_0=1,\; \bar{w}_0=w_0 = 0,\; \kappa=1-O(\mu).  \nn
	\)\\
	{\bf \small repeat for $i\geq 1$}:
	\begin{align}
	\w_i=&\; \w_{i-1} -\mu \widehat{g}(\w_{i-1})\\
	S_i=&\;\kappa S_{i-1} + 1\\
	\bar{\w}_i=&\;\left(1-\frac{1}{S_i}\right)\bar{\w}_{i-1}\;+\;\frac{1}{S_i} \w_i\hspace{5mm}
	\end{align}
	{\bf \small end}\\[-2mm]
	\rule{0.49\textwidth}{1pt}\vspace{-3mm}

	\subsection{Interpretation of Results}
	\noindent The results derived in this section highlight several important facts that we would like to summarize:
	\begin{enumerate}
		\item[(1)] First, it has been observed in the optimization literature that sub-gradient descent iterations can perform poorly in {\em deterministic} problems (where $J(w)$ is known). Their convergence rate is $O(1/\sqrt{i})$ under convexity and $O(1/i)$ under strong-convexity \cite{nesterov2004introductory} when decaying step-sizes, $\mu(i)=1/i$, are used to ensure convergence \cite{shalev2011pegasos}. Our arguments show that the situation is different in the context of {\em stochastic} optimization when true sub-gradients are approximated from streaming data. By using {\em constant} step-sizes to enable continuous learning and adaptation, the sub-gradient iteration is now able to achieve exponential convergence at the rate of $O(\alpha^i)$ for some $\alpha=1-O(\mu)$.
	
		\item[(2)] Second, of course, this substantial improvement in convergence rate comes at a cost, but one that is acceptable and controllable. Specifically, we cannot guarantee convergence of the algorithm to the global minimum value, $J(w^{\star})$, anymore but can instead approach this optimal value with high accuracy in the order of $O(\mu)$, where the size of $\mu$ is under the designer's control and can be selected as small as desired.
		
		\item[(3)] Third, this performance level is sufficient in most cases of interest because, in practice, one rarely has an infinite amount of data and, moreover, the data is often subject to distortions not captured by any assumed models. It is increasingly recognized in the literature that it is not always necessary to ensure exact convergence towards the  optimal solution, $w^{\star}$, or the minimum value, $J(w^{\star})$, because these optimal values may not reflect accurately the true state due to modeling error. For example, it is explained in the works \cite{bousquet2008tradeoffs,bottou2012stochastic,bottou2010large} that it is generally unnecessary to reduce the error measures below the statistical error level that is present in the data.
	\end{enumerate}
	
	\section{Applications: Single Agent Case}
	We now apply the results of the previous analysis to several cases in order to illustrate that {\it stochastic} sub-gradient constructions can indeed lead to good performance. \\

	\noindent {\bf Example 5 (LASSO problem).} We run a simulation with $\mu=0.001$, $\delta = 0.002$, and $M=100$. Only two entries in $w^{\circ}$ are assumed to be nonzero. The regression vectors and noise process
	$\{\h_i,\n(i)\}$ are both generated according to  zero-mean normal distributions with variances $R_h=I$ and $\sigma_n^2=0.01,$ respectively.
	Thus under this case, we have
	\be
	e^{2} = 2c^2 = 2, \;\;\;f^{2} = 2d^2 = 8\delta^2M
	\ee
	It then follows that
	\begin{align}
	&\hspace{-5mm}\lim_{L\to\infty}\,\Ex\,J^{\rm lasso}(\bar{\w}_L)-J^{\rm lasso}(w^{\star})\nn\\
	&\;\;\;\;\;\leq  4\mu\delta^2M+\frac{\mu}{2}\sigma_n^2 \Tr(R_h)+\mu a\|w^o-w^\star\|^2.
	\end{align}

	In order to verify this result, 
	From the optimality condition, $0\in\partial J(w^\star)$, it is easy to conclude that\cite{donoho1994ideal}
	\be
	w^\star = \mathcal{S}_\delta(w^{\circ}),
	\ee
	where the symbol $\mathcal{S}_\delta$ represents the soft-thresholding function with parameter $\delta$, i.e.,
	\be
	\mathcal{S}_\delta(x) ={\rm sgn}(x)\cdot \max\{0,|x| - \delta\}.
	\ee
	Figure~\ref{fig.SLMS} plots the evolution of the excess-risk curve,
	$\Ex\,J^{\rm lasso}(\bar{\w}_L)-J^{\rm lasso}(w^{\star})$, obtained by averaging over $50$ experiments. The figure compares the performance of the standard LMS solution:
	\be
	\w_{i} = \w_{i-1} + \mu \h_i(\bgamma(i)-\h_i\tran\w_{i-1}),
	\ee
	against the sparse sub-gradient version\cite{duttweiler2000proportionate,kopsinis2011online,chen2009sparse}:
	\be
	\w_{i} = \w_{i-1} + \mu \h_i(\bgamma(i)-\h_i\tran\w_{i-1})-\mu\delta\cdot{\rm sgn}(\w_{i-1}).
	\ee
	and the subgradient algorithm with exponential smoothing from this article (listed below).\\[-2mm]
		\rule{0.49\textwidth}{1pt}
	{\bf \small Stochastic sparse LMS  with exponential smoothing}\vspace*{-3mm}\\
	\rule{0.49\textwidth}{1pt}
	{\bf \small Initialization}:
	\(
	S_0=1,\; \bar{w}_0=w_0 = 0  \nn
	\)\\
	{\bf\small Repeat}:
	\begin{align}
	\w_{i} =&  \w_{i-1} + \mu \h_i(\bgamma(i)-\h_i\tran\w_{i-1})-\mu\delta\cdot{\rm sgn}(\w_{i-1}) \label{slms.impl1}\\
	S_i =&\;\kappa S_{i-1} +1\;\; \label{slms.impl2}\\
	\bar\w_{i} =& \;\left(1-\frac{1}{S_i}\right)\bar{\w}_{i-1}\;+\;\frac{1}{S_i}\w_i\label{slms.impl3}
	\end{align}
	{\bf\small End} \vspace{-1.5mm}\\
	\rule{0.49\textwidth}{1pt}
	In the simulation, we set the parameter $\kappa=0.999$.
	It is observed that the stochastic sub-gradient implementation satisfies the bound predicted by theory. Observe that while the smoothed implementation is able to attain better MSD performance, as predicted by theory, the convergence during the initial stages of learning is slowed down. 

	\begin{figure}[h]
		\epsfxsize 7.8cm \epsfclipon
		\begin{center}
			\leavevmode \epsffile{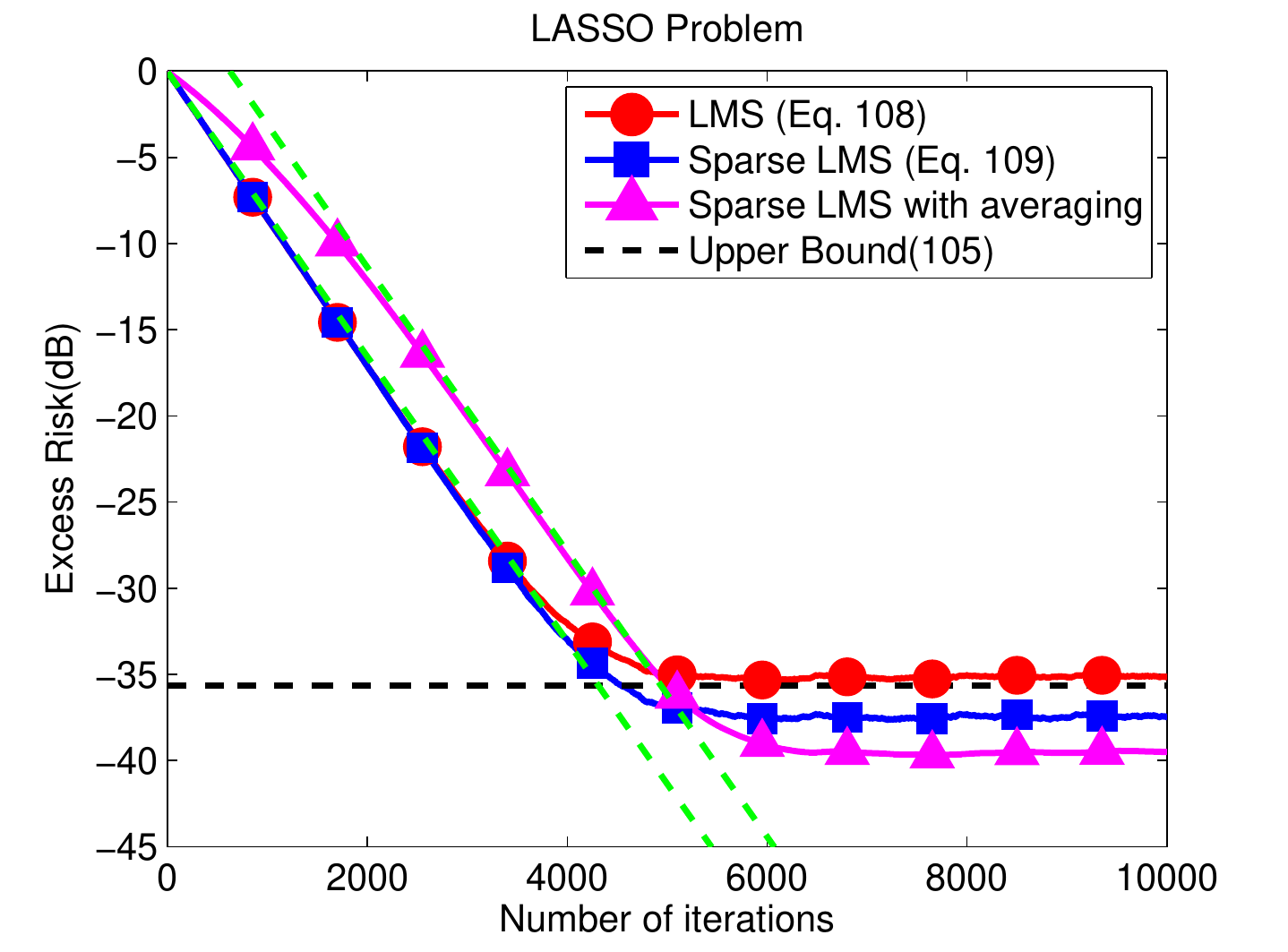} \caption{{ LASSO problem. The excess-risk curves, i.e., \big($\Ex\,J({\w}_L)-J(w^{\star})$\big), for LMS, sparse-LMS, and sparse-LMS with weighting are obtained by averaging over $50$ experiments. The two green dot lines are the referring convergence rate.\vspace{-0.7cm}}}\label{fig.SLMS}
		\end{center}
	\end{figure}
	
	$\hfill \Box$
	
	\noindent {\bf Example 6 (SVM problem).} For the SVM problem, we can conclude from (\ref{lkadj13123}) that
	\begin{align}
	\lim_{L\to\infty}\,\Ex\,J^{\rm svm}(\bar{\w}_L)-J^{\rm svm}(w^{\star})&\,\leq\,\frac{\mu}{2}(2d^2+\sigma^2)\leq \frac{9\mu}{2}\Tr(R_h)\nn
	\end{align}
	Actually, for the SVM construction, we can obtain another upper bound than the one provided by Corollary~\ref{coolad.12}; this is because we can exploit the special structure of the
	SVM cost to arrive at
	\bq
	\lim_{L\to\infty}\,\Ex\,J^{\rm svm}(\bar{\w}_L)-J^{\rm svm}(w^{\star})\hspace{18mm}\nn\\
	\hspace{11mm}\leq	\mu  \big(\rho^2 \|w^\star\|^2+\rho + {\rm Tr}(R_h)/2 \big)\label{tighter.svm},
	\eq	
	with convergence rate $\alpha = 1- 2\mu\rho+2\mu^2\rho^2$. 
	The proof is provided in Appendix \ref{apd.SVMtight}. We list below the stochastic subgradient SVM implementation with exponential smoothhing using the above  value for $\kappa$. \\[-2mm]
	\rule{0.49\textwidth}{1pt}
	{\bf \small Stochastic subgradient SVM  with exponential smoothing}\vspace*{-3mm}\\
	\rule{0.49\textwidth}{1pt}
	{\bf \small Initialization}:
	\(
		S_0=1,\; \bar{w}_0=w_0 = 0  \nn
	\)\\
	{\bf\small Repeat}:
	\begin{align}
		\w_{i} =& \;(1-\mu\rho)\w_{i-1} - \mu\bgamma(i)\h_i\mathbb{I}[\bgamma(i)\h_i\tran\w_{i-1}\leq 1] \label{svm.impl1}\\
		S_i =&\;(1- 2\mu\rho+2\mu^2\rho^2) S_{i-1} +1\;\; \label{svm.impl2}\\
		\bar\w_{i} =& \;\left(1-\frac{1}{S_i}\right)\bar{\w}_{i-1}\;+\;\frac{1}{S_i}\w_i\label{svm.impl3}
	\end{align}
	{\bf\small End} \vspace{-1.5mm}\\[-1mm]
	\rule{0.49\textwidth}{1pt}
	\noindent We compare the performance of the stochastic sub-gradient SVM implementation listed above (with all variables initialized to zero) against LIBSVM (a popular SVM solver that uses dual quadratic programming) \cite{CC01a}. The test data is obtained from the LIBSVM website\footnote{\url{http://www.csie.ntu.edu.tw/~cjlin/libsvmtools/datasets/binary.html}} and also from the UCI dataset\footnote{\url{http://archive.ics.uci.edu/ml/}}. We first use the Adult dataset after preprocessing \cite{platt1999fast} with 11,220 training data and 21,341 testing data in 123 feature dimensions. The purpose of the learning task is to use the personal information such as age, education, occupation, and race to predict whether a person makes over $50$K a year.
	To ensure a fair comparison, we use linear LIBSVM with the exact same parameters as the sub-gradient method. Hence, we choose $C=5\times 10^2$ for LIBSVM, which corresponds to $\rho = \frac{1}{C}=2\times 10^{-3}$. We also set $\mu=0.05$.
	\begin{figure}[h]
		\epsfxsize 7.7cm \epsfclipon
		\begin{center}
			\leavevmode \epsffile{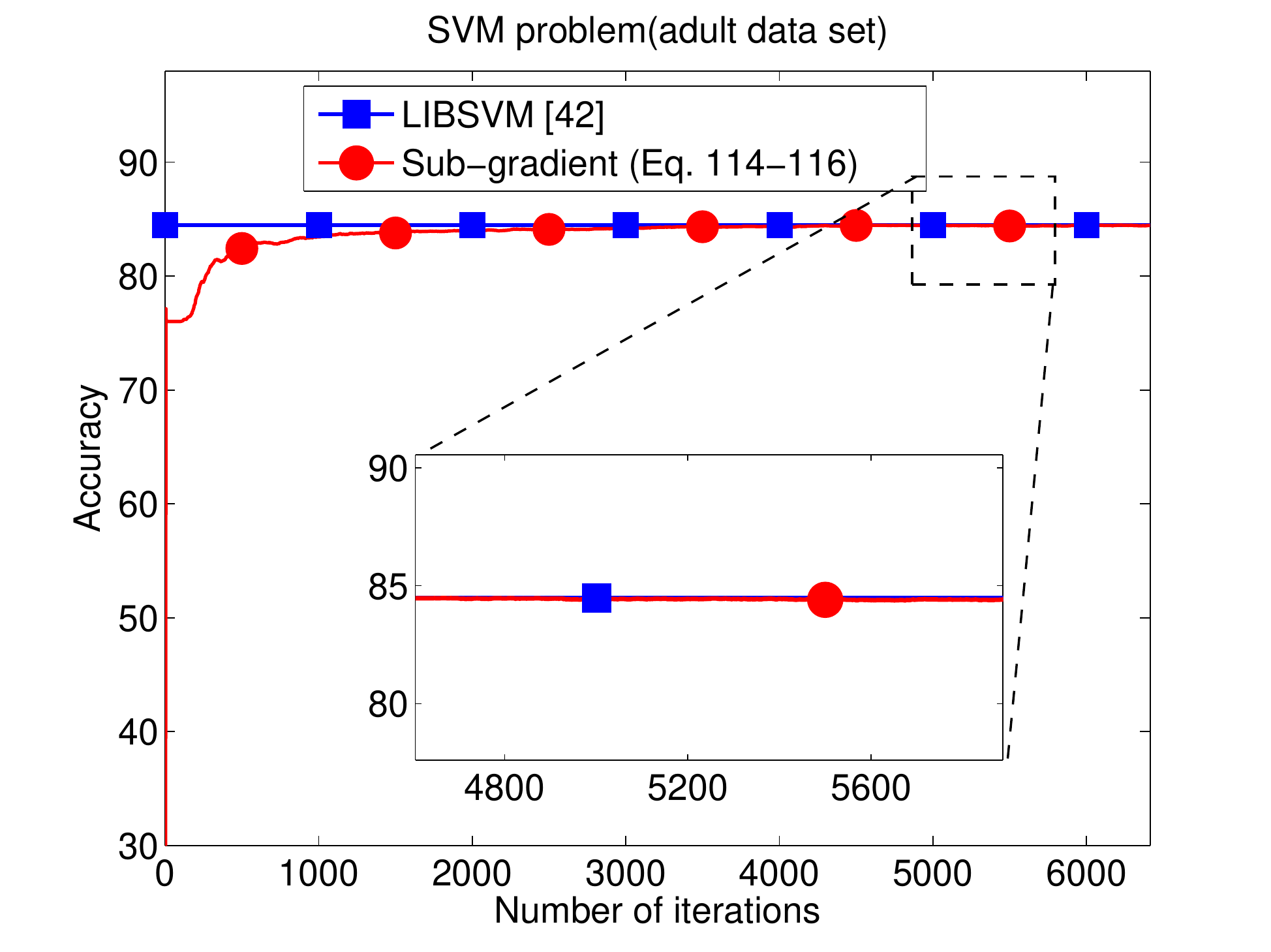} \caption{{ SVM solvers applied to the Adult data set. Comparison of the performance accuracy, percentage of correct prediction over test dataset, for LIBSVM\cite{CC01a} and a stochastic sub-gradient implementation (\ref{svm.impl1})-\ref{svm.impl3}).\vspace{-5mm}}}\label{fig.SVM_adult}
		\end{center}
	\end{figure}
	
	We can see from Fig.~\ref{fig.SVM_adult} that the stochastic sub-gradient algorithm is able converge to the performance of LIBSVM quickly. Since we only use each data point once, and since each iteration is computationally simpler, the sub-gradient implementation ends up being computationally more efficient.  We also examine the performance of the sub-gradient SVM solver on another large-scale dataset, namely, the Reuters Corpus Volume I (RCV1) data with 20242 training data and 253843 testing data consisting of 47236 feature dimensions. The RCV1 dataset uses cosine-normalized, log TF-IDF ({ term frequency-inverse document frequency}) vectors representations for newswire stories to predict their categories. The chosen parameters are $C=1\times 10^5, \mu =0.2$. The performance is shown in Fig.~\ref{fig.SVM_RCV1}. $\hfill \Box$\vspace{-3mm}
	\begin{figure}[h]
		\epsfxsize 7.8cm \epsfclipon
		\begin{center}
			\leavevmode \epsffile{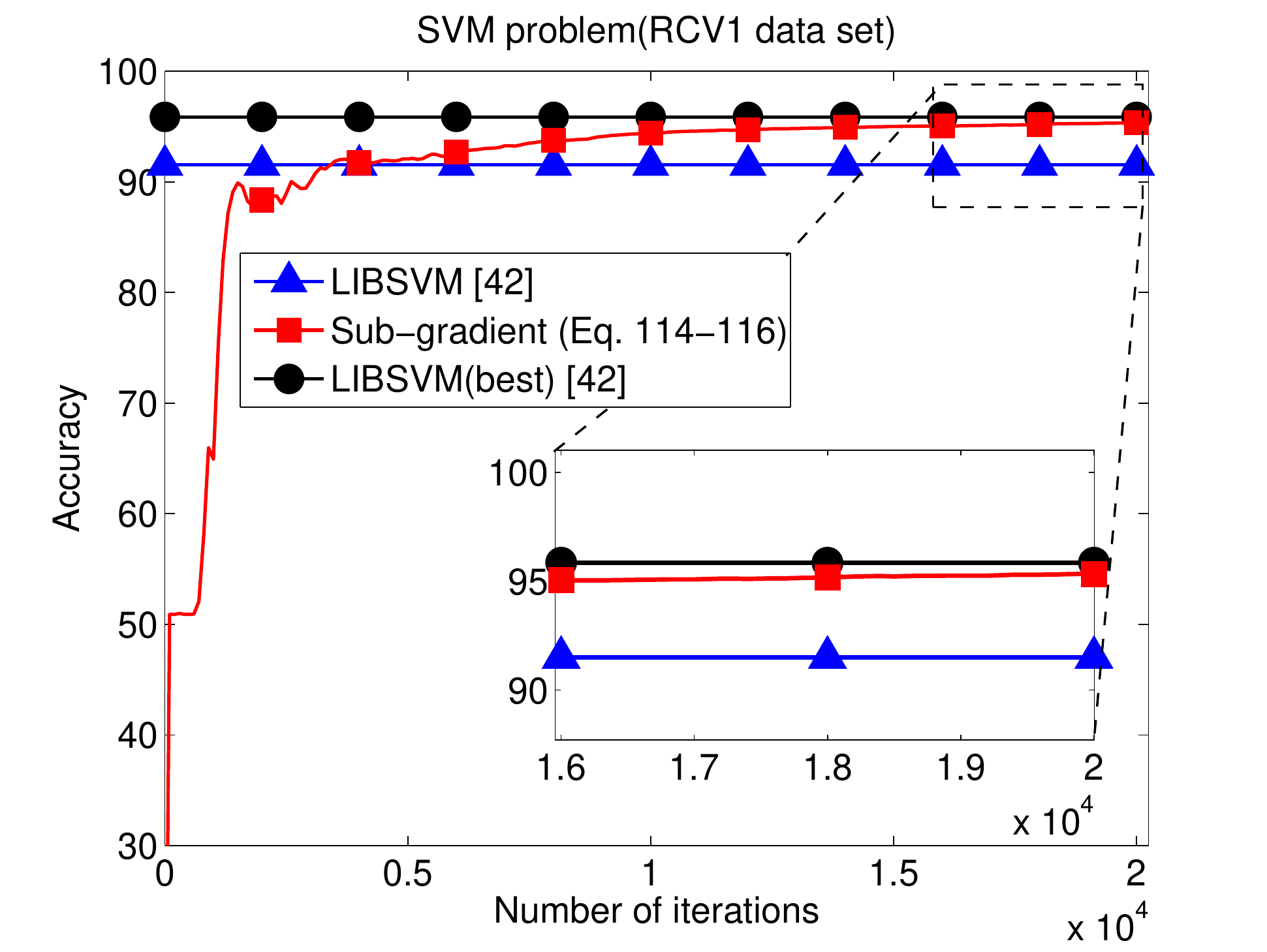} \caption{{ SVM solvers applied to the RCV1 data set. Comparison of the performance accuracy, percentage of correct prediction over test dataset, for LIBSVM\cite{CC01a} and the stochastic sub-gradient implementation (\ref{svm.impl1})-(\ref{svm.impl3}). The blue line for  LIBSVM is generated by using the same parameters as the sub-gradient implementation, while the black line is determined by using cross validation. The difference between both lines is because LIBSVM achieves higher accuracy when setting $\rho$ to a large value, say, around the value of one. In comparison, from (\ref{tighter.svm}) we know that sub-gradient methods need a small $\rho$ to achieve higher accuracy. \vspace{-5mm}}}\label{fig.SVM_RCV1}
		\end{center}
	\end{figure}

	\noindent {\bf Example 7 (Image denoising problem).} We next illustrate how the stochastic sub-gradient implementation can match  the performance of some sophisticated techniques for image denoising, such as the FISTA algorithm. This latter technique solves the denoising problem by relying on the use of accelerated proximal methods applied to a dual problem \cite{beck2009fast,beck2009fast2}.
	
	Specifically, one classical formulation for the image denoising problem with total-variation regularization  involves seeking an image (or matrix) that minimizes the following deterministic cost\cite{rudin1992nonlinear}:
	\be
	\min_\mathcal{I}\; \frac{1}{2}\|\mathcal{I}- \mathcal{I}_{\rm noisy}\|^2_F + \lambda\cdot {\rm TV}(\mathcal{I}), \label{eq.TVproblem}
	\ee
	where $\lambda>0$ is a regularization factor. Moreover, the term $\mathcal{I}$ denotes some rectangular or square image that we wish to recover, say, of size $N\times N$, and $\mathcal{I}_{\rm noisy}$ refers to the available noisy measurement of the true image:
	\be
	\mathcal{I}_{\rm noisy}\;=\;\mathcal{I}^o\;+\;\mbox{\rm noise}, \label{image.noise}
	\ee
	where the noise term refers to a zero-mean perturbation.  The notation $\|\cdot\|_F$ denotes the Frobenious norm of its matrix argument, and the operation $\rm TV(\cdot)$ stands for a total-variation computation, which is defined as follows:\footnote{Here, we only consider the discrete $\ell_1$-based anisotropic TV and neglect the boundary modification.}
	\be
	{\rm TV}(\mathcal{I})\hspace{-0.5mm} \define\hspace{-0.5mm}  \sum_{m,n} |\mathcal{I}(m,n)-\mathcal{I}(m+1,n)| + |\mathcal{I}(m,n)-\mathcal{I}(m,n+1)|.\nn
	\ee
	The total variation term essentially encourages the difference between the image and some of its shifted versions to remain nearly sparse. We may also formulate a stochastic version of the denoising problem by considering instead:
	\be
	\min_{{\mathcal{I}}}\; \frac{1}{2}\Ex\|{{\mathcal{I}}}-{\bm{\mathcal{I}}}_{\rm noisy}\|^2_F + \lambda\cdot {\rm TV}({{\mathcal{I}}}) \label{eq.TVproblem.2},
	\ee
	where the expectation is now over the randomness in the noise used to generate the noisy image (here we only consider the synthesis case).
	The sub-gradient of the Total Variation term is straightforward to compute. For illustration purposes, we evaluate the sub-gradient at some arbitrary point $(m_0,n_0)$. Expanding the summation and separating the terms related to point $(m_0,n_0)$, we obtain:
	\begin{align}
	\hspace{-0mm}\rm TV(\mathcal{I})\hspace{-0.5mm} = & |\mathcal{I}(m_0,n_0)\hspace{-1mm}-\hspace{-1mm}\mathcal{I}(m_0\hspace{-1mm}+\hspace{-1mm}1,n_0)|\hspace{-1mm}+\hspace{-1mm} 	|\mathcal{I}(m_0,n_0)\hspace{-1mm}-\hspace{-1mm}\mathcal{I}(m_0,n_0\hspace{-1mm}+\hspace{-1mm}1)|\nn\\
	&\;+|\mathcal{I}(m_0\hspace{-1mm}-\hspace{-1mm}1,n_0)\hspace{-1mm}-\hspace{-1mm}\mathcal{I}(m_0,n_0)|\hspace{-1mm}\nn\\
	&\;{}+|\mathcal{I}(m_0,n_0\hspace{-1mm}-\hspace{-1mm}1)\hspace{-1mm}-\hspace{-1mm}\mathcal{I}(m_0,n_0)|+ \rm rest,\label{tv.diff}
	\end{align}
	where the $rest$ variable refers to terms that do not contain the variable $\mathcal{I}(m_0,n_0)$. Computing the sub-gradient with respect to $\mathcal{I}(m_0,n_0)$ will generate four terms with the sign function as in
	the LASSO problem. It is then clear that the stochastic sub-gradient implementation in this case is
	given by:\vspace{-2mm}
	\be
	{\bm{\mathcal{I}}}_{i} ={\bm{\mathcal{I}}}_{i-1} - \mu\Big({\bm{\mathcal{I}}}_{i-1}-{\bm{\mathcal{I}}}_{\rm noisy} + \sum_{j=1}^4 \lambda\cdot {\rm sgn}({\bm{\mathcal{I}}}_{i-1}-{\bm{\mathcal{I}}}_{i-1}^j) \Big)
	\label{one.laldk123}\ee
	where  ${\bm{\mathcal{I}}}_i$ represents the recovered image at iteration $i$,  ${\bm{\mathcal{I}}}_i^{1}$ represents shifting the image to the left by one pixel, while ${\bm{\mathcal{I}}}_i^{2}, {\bm{\mathcal{I}}}_i^{3}, {\bm{\mathcal{I}}}_i^{4}$ represent shifting the image  to the right, up, and down by one pixel, respectively. We observe that recursion (\ref{one.laldk123}) now iterates repeatedly over the {\em same} single image, ${\bm{\mathcal{I}}}_{\rm noisy}$. Accordingly, in this example, the stochastic gradient noise does not vary over time, i.e.,\vspace{-3mm}
	
	\begin{table*}[t]
		\centering
		\caption{\rm Comparison between the stochastic sub-gradient method (\ref{one.laldk123}) and FISTA\cite{beck2009fast} over the KODIM test image set (c.f. footnote 4). All test images are subject to additive zero-mean Gaussian noise with standard variance 0.1 (with respect to image values in the range $[0,1]$). We set $\lambda =0.08, \mu=0.002$ and 300 max iterations for sub-gradient methods. For different values of $\lambda$ and $\mu$, the results will be different, but the algorithms will perform similarly when $\mu$ is chosen properly.  The results in the table show that the sub-gradient implementation can, in general, achieve similar or higher PSNR in shorter time.\vspace{-3mm}}
		{\small
		\begin{tabular}{c|c||cccccccccc}
			\hline
			& Test Image \cellcolor[gray]{0.8}& kodim1\cellcolor[gray]{0.8} & kodim5 \cellcolor[gray]{0.8}& kodim7\cellcolor[gray]{0.8} & kodim8\cellcolor[gray]{0.8}& kodim11\cellcolor[gray]{0.8}& kodim14 \cellcolor[gray]{0.8}& kodim15 \cellcolor[gray]{0.8}& kodim17 \cellcolor[gray]{0.8}& kodim19
			\cellcolor[gray]{0.8}& kodim21 \cellcolor[gray]{0.8}\\
			\hhline{============}
			\hspace{-0.15cm}\multirow{2}{*}{PSNR(dB)}\hspace{-0.1cm}
			& Sub-gradient
			& \bf 25.19 &\bf 25.18 &\bf 29.43 &\bf 24.59 &\bf 27.80 &\bf 30.32 &\bf 30.32&\bf 29.38 &\bf 27.53 &\bf 27.29  \\
			\cline{2-12}
			
			& FISTA
			& 24.90 & 24.87 & 29.14 & 24.26 & 27.59 & 30.25  & 30.25 & 29.17 & 27.23 & 27.03\\
			\hlinewd{1.2pt}
			
			\hspace{-0.15cm}\multirow{2}{*}{Time(s)}\hspace{-0.1cm}
			& Sub-gradient
			&\bf 8.88 &\bf 9.33 &\bf 8.50 & 8.46 &\bf 9.00 &\bf 8.7  &\bf 9.07 &\bf 9.50 & 9.60 &\bf 9.45\\
			\cline{2-12}
			
			& FISTA
			& 9.16 & 9.78 & 10.15 &\bf 8.22 & 10.24 & 9.19 & 10.13 & 9.98 &\bf 9.47 & 9.62 \\
			\hline
			
%
%
%

		\end{tabular}
		\label{table.11}
	}

	\end{table*}
\be
	s_i({\bm{\mathcal{I}}}_i)\;=\; \mathcal{I}^{o} - {\bm{\mathcal{I}}}_{\rm noisy}, \quad \forall i.\label{gfewf94q}
\ee
	It is obvious from (\ref{image.noise}) that the gradient noise has zero mean and bounded variance. Therefore, Assumption \ref{asm.noise} still holds. Table~\ref{table.11} lists performance results using the Kodak image suite\footnote{\url{http://r0k.us/graphics/kodak/}}. The table lists two metrics. The first metric is the PSNR:\vspace{-1mm}
	\be
	{\rm PSNR}(\mathcal{I}) \define 10\times\log{\frac{(255)^2}{\mbox{\rm MSE}(\mathcal{I})}},
	\ee
	where $\mbox{\rm MSE}$ represents the mean-square-error,
	\be
		{\rm MSE}(\mathcal{I}) \define  \frac{1}{MN}\sum_{m,n} |\mathcal{I}(m,n) - \mathcal{I}^o(m,n)|^2 
	\ee
	where $M$ and $N$ are the length and width of image and the second metric is the execution time. For a  fair comparison, we used similar un-optimized MATLAB codes\footnote{Code for FISTA is available at \url{http://iew3.technion.ac.il/~becka/papers/tv\_fista.zip}.} under the same computer environment. The table shows that the sub-gradient implementation can achieve comparable or higher PSNR values in shorter time. Clearly, if we vary the algorithm parameters, these values will change. However, in general, it was observed in these experiments that  the sub-gradient implementation succeeds in matching the performance of FISTA reasonably well.
	$\hfill \Box$
	\vspace{-5mm}

	\begin{figure*}[h]
		\epsfxsize 10cm \epsfclipon
		\begin{center}
		\leavevmode
			\includegraphics[scale=0.61]{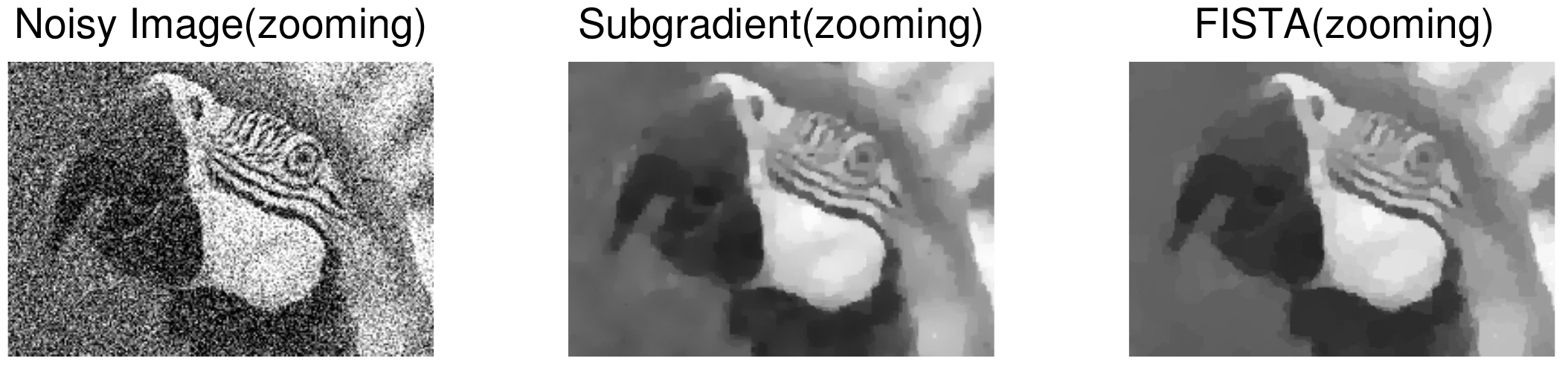} \vspace{-3mm}
			\caption{{ Comparison of the performance of FISTA and sub-gradient implementations on  test image  Kodim 23 (zoom-in) under low PSNR (13dB). The result of the sub-gradient implementation is almost indistinguishable from the result of FISTA. }}\label{fig.image denoise}\vspace{-4mm}
		\end{center}
	\end{figure*}
	
	\appendices
	\section{\sc Derivation of the Another SVM Bound (\ref{tighter.svm})} \vspace{-1mm}\label{apd.SVMtight}
	We assumed in the exposition leading to Theorem~\ref{them.121} and Corollary~\ref{coolad.12} that the sub-gradient vectors and the variance of the gradient noise satisfy affine-like forms separately --- see
	(\ref{usadkh.13lk1l3k}) and (\ref{assump.subGradient}). For the case of the SVM problem, a joint bound can be derived as follows. First, note that
	\begin{align}
	&\hspace{-3mm}\Ex \left[\,\left\|g^{\rm svm}(\w_{i-1})+\s_{i}(\w_{i-1})\right\|^2  \,| \,{\bm {\cal F}}_{i-1}\,\right]\nn\\
	=&\Ex\left[\,\left\|\rho \w_{i-1}+\bgamma(i) \h_i\, \mathds{I} [\bgamma(i)\h_i\tran \w_{i-1} \leq 1]  \right\|^2  \,| \,{\bm {\cal F}}_{i-1}\,\right]\nn\\
	=& \rho^2 \|\w_{i-1}\|^2 + 2\rho \Ex \bgamma(i) \h_i\tran \w_{i-1}\,\mathds{I} [\bgamma(i)\h_i\tran \w_{i-1} \leq 1]   \,| \,{\bm {\cal F}}_{i-1}] + \nn\\
	& \quad \Ex \left[\bgamma^2(i) \h_i\h_i\tran\, \mathds{I} [\bgamma(i)\h_i\tran \w_{i-1} \leq 1] \,| \,{\bm {\cal F}}_{i-1}\right] \nn\\[-2mm]
	\stackrel{(a)}{\leq}& \rho^2 \|\w_{i-1}\|^2 + 2\rho +\Ex [\h_i\tran\h_i\, \mathds{I} [\bgamma(i)\h_i\tran \w_{i-1} \leq 1] \,| \,{\bm {\cal F}}_{i-1}] \nn\\
	\;\leq\;& \rho^2 \|\w_{i-1}\|^2 + 2\rho+\Tr(R_h) \nn\\
	\;\stackrel{(b)}{\leq}\;& 2\rho^2 \|\widetilde\w_{i-1}\|^2 +2\rho^2\|w^\star\|^2+ 2\rho+\Tr(R_h)  \label{eq.SVMtight},
	\end{align}
	where in step (a) we used the facts that $\bgamma^2(i)=1$ and
	\be
	\bgamma(i) \h_i\tran \w_{i-1}\mathds{I} [\bgamma(i)\h_i\tran \w_{i-1} \leq 1]   \,| \,{\bm {\cal F}}_{i-1}]\leq 1,
	\ee
	while step (b) follows from Jensen's inequality by adding and subtracting $w^{\star}$ to $\w_{i-1}$. We therefore conclude that
	\begin{align}
	&\hspace{-2mm}\Ex [\|g^{\rm svm}(\w_{i-1})+\s_{i}(\w_{i-1})\|^2  \,| \,{\bm {\cal F}}_{i-1}]\nn\\
	&\;\hspace{4mm}\leq 2\rho^2 \|\widetilde\w_{i-1}\|^2 +2\rho^2\|w^\star\|^2+ 2\rho+\Tr(R_h).
	\end{align}
	We now use this result to expand the first line of (\ref{subckakjdk13lk13}):
	\begin{align}
	&\hspace{-4mm}\Ex[\|\widetilde{\w}_{i}\|^2\,|\,\bm{\cal F}_{i-1}\,] 
	\nn\\
		=&\; \Ex[\,\|\widetilde{\w}_{i-1}+\mu g^{\rm svm}(\w_{i-1}) + \mu \s_{i}(\w_{i-1})\|^2 \,|\,\bm{\cal F}_{i-1}\,]\nn\\
		=&\;\|\widetilde{\w}_{i-1}\|^2+2\mu g^{\rm svm}(\w_{i-1})\tran \widetilde{\w}_{i-1}\nn\\
		&\; {}+  \mu^2\Ex[\|g^{\rm svm}(\w_{i-1})+\s_{i}(\w_{i-1})\|^2\,|\,\bm{\cal F}_{i-1}\,]\nn\\
		\leq&\; \|\widetilde{\w}_{i-1}\|^2+2\mu g^{\rm svm}(\w_{i-1})\tran \widetilde{\w}_{i-1}\nn\\
		&\;\;{}+ \mu^2\big(2\rho^2 \|\widetilde\w_{i-1}\|^2 +2\rho^2\|w^\star\|^2+ 2\rho+\Tr(R_h)\big).\label{eq.133}
	\end{align}
	Indeed, using strong convexity, we have
		\begin{align}
			\hspace{-2mm}g^{\rm svm}(\w_{i-1})\tran \widetilde{\w}_{i-1}\leq J(w^\star) - J(\w_{i-1}) - \frac{\rho}{2} \|\widetilde{\w}_{i-1}\|^2
		\end{align}
		Substituting into (\ref{eq.133}), we get:
		\begin{align}
		&\hspace{-3mm}2\mu (J(\w_{i-1}) - J(w^\star) ) \leq (1 -\mu\rho + 2\mu^2\rho^2)\|\widetilde{\w}_{i-1}\|^2\\[-.6mm]
		&-\Ex[\|\widetilde{\w}_{i}\|^2\,|\,\bm{\cal F}_{i-1}\,] +2\mu^2\rho^2\|w^\star\|^2+ 2\mu^2\rho+\mu^2\Tr(R_h)\nn
		\end{align}
	Now following the same steps after (\ref{kjad567llk13lk1l3}), we arrive at the tighter bound  (\ref{tighter.svm}).\vspace{-3mm}
	
\bibliographystyle{IEEEbib}
\bibliography{ref_sbg}
\end{document}